\documentclass[11pt]{article}

\usepackage[margin=1in]{geometry}

\usepackage[T1]{fontenc}
\usepackage{times}
\usepackage[hyphens]{url}
\usepackage{graphicx}
\usepackage{natbib}
\usepackage{caption}
\usepackage{algorithm}
\usepackage{algpseudocode}

\usepackage{amsmath}
\usepackage{amssymb}
\usepackage{amsthm}
\newtheorem{lemma}{Lemma}
\newtheorem{remark}{Remark}
\usepackage{subcaption}
\usepackage{booktabs}
\usepackage{multirow}
\usepackage{xcolor}
\usepackage{textcomp}
\usepackage{xurl}
\usepackage{fontawesome5}
\usepackage{wrapfig}
\usepackage{needspace}
\usepackage{hyperref}
\hypersetup{
  colorlinks=true,
  linkcolor=blue!50!black,
  citecolor=blue!50!black,
  urlcolor=blue!50!black
}

\definecolor{cIncorrect}{rgb}{0.0,0.55,0.55} 
\definecolor{cGamma}{rgb}{0.90,0.49,0.13}    
\definecolor{cRhythm}{rgb}{0.55,0.30,0.78}   

\newcommand{\drifttitleblue}[1]{{\color[rgb]{0.18,0.48,0.78}#1}}

\newcommand{\sg}{\operatorname{sg}}

\definecolor{impgreen}{rgb}{0.0,0.55,0.0}
\newcommand{\imp}[1]{\ensuremath{_{\textcolor{impgreen}{+#1}}}}

\newcommand{\includefig}[2][]{%
  \IfFileExists{#2}{%
    \includegraphics[#1]{#2}%
  }{%
    \fbox{\parbox[c][0.22\textheight][c]{0.9\linewidth}{\centering\ttfamily Missing figure:\\ \detokenize{#2}}}%
  }%
}

\title{\drifttitleblue{DRIFT}: \drifttitleblue{D}ifficulty \drifttitleblue{R}outing Self-D\drifttitleblue{I}stillation with Rhythm-Gated Exploration and Success Bu\drifttitleblue{F}fer \drifttitleblue{T}raining}

\author{\normalfont 
Haisen Luo$^{3,*}$, 
Yiwei Liu$^{2,*}$, 
Haoning Wang$^{1,*}$, 
Dan Liu$^{3,*}$, Junxi Yin$^{3,\dagger}$,\\
Haotian Wang$^{3}$, Lei Zhang$^{3}$, Xiaoyu Tian$^{3}$, Shuaiting Chen$^{3}$, Yuansheng Song$^{3}$, Baoyan Guo$^{3}$,\\
Xiongfei Yan$^{3}$, Bolan Yang$^{3}$, Chengwei Liu$^{3}$, Ming Cui$^{3}$, Jiong Chen$^{3,\ddagger}$\\
\\
$^{1}$Tsinghua University \quad
$^{2}$ENS Paris-Saclay \quad
$^{3}$Beike Language and Intelligence\\
{\small\bfseries $^{*}$Equal contribution. \quad
$^{\dagger}$Project leader. \quad
$^{\ddagger}$Corresponding author.}\\
\\
\smallskip
\faGithub~{\fontfamily{ppl}\selectfont \textbf{Code:}}~\url{https://github.com/LianjiaTech/drift}.\\
\faEnvelope[regular]~\texttt{
whn22@mails.tsinghua.edu.cn,
sunsetlyw2002@gmail.com,
}\\
\texttt{\{luohaisen002,liudan190,yinjunxi001,chenjiong\}@ke.com}
}

\begin{document}

\maketitle

\begin{abstract}
Enabling large language models to achieve stable self-improvement without external expert supervision remains a central challenge in complex reasoning tasks. Existing self-distillation and reinforcement learning methods lack explicit mechanisms for tracking problem-level learning progress and adapting optimization strategies accordingly. Consequently, training may over-optimize easy problems, receive weak supervision from hard problems, and fail to sufficiently explore borderline cases. To resolve these issues, we propose DRIFT, an online self-evolution policy optimization framework for large language models. DRIFT regulates the model's self-improvement process through the joint use of Difficulty Routing and Rhythm Gating. The former identifies the model's learning state at the problem level and dynamically allocates self-distillation and reinforcement learning signals, while the latter refines policy updates at the token level by selectively amplifying verified policy innovations at unresolved decision boundaries. By further incorporating a success buffer and a two-stage curriculum learning strategy, DRIFT preserves high-quality historical experience while progressively guiding the model from reliable behavior acquisition toward stable policy evolution. Evaluated across five benchmarks and three model scales, DRIFT surpasses the peak performance of both GRPO and SDPO across all evaluated metrics. On the average score over the five benchmarks, DRIFT achieves 79.5\%, outperforming GRPO by 9.5\% and SDPO by 7.5\%, establishing a new state-of-the-art result. Notably, on ToolUse, DRIFT reaches an accuracy of 79.2\%, improving over GRPO by 13.5\% and SDPO by 10.7\%, setting a new state-of-the-art and substantially outperforming all concurrent methods.
\end{abstract}

\section{Introduction}

Recently, large language models (LLMs) have made remarkable progress on complex tasks such as mathematical reasoning, scientific problem solving, and tool use \cite{jaech2024learning,guo2025deepseekr1,deepseekai2026deepseekv4,comanici2025gemini25,olmo2025olmo3}. Beyond scaling pre-training, feedback-based reinforcement learning and self-distillation have emerged as important approaches for strengthening LLM reasoning\cite{shao2024deepseekmath,gu2023minillm,hubotter2026reinforcement,zhao2026opsd,jin2026entropyopd,fu2026revisitingopd,pan2026rlcsd,li2026rethinkingopd}. By learning from trajectories the model produces itself, these methods reduce the dependence on human expert demonstrations and offer a scalable route to self-evolution.

\begin{wrapfigure}[14]{r}{0.52\textwidth}
    \vspace{-1.0\baselineskip}
    \centering
    \includegraphics[width=\linewidth]{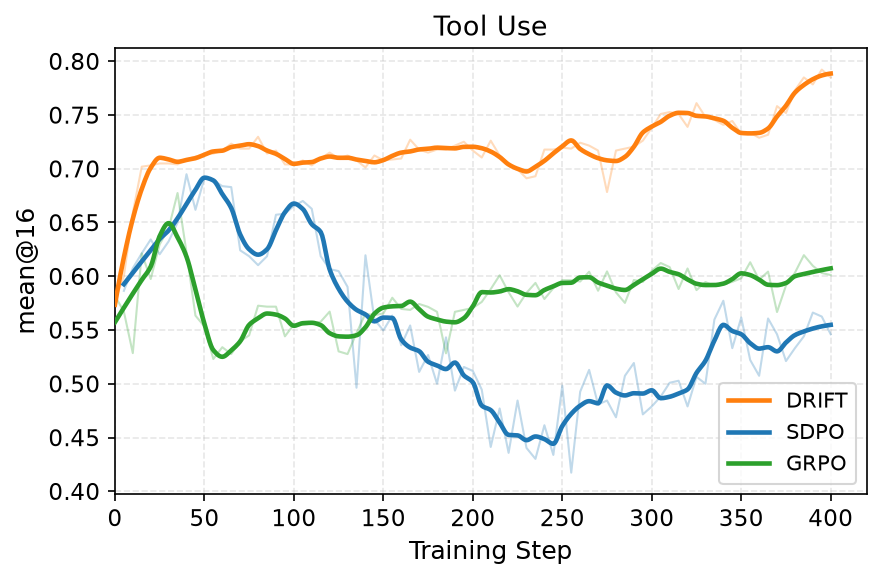}
    \captionsetup{font=scriptsize,skip=2pt}
    \caption{DRIFT substantially outperforms SDPO and GRPO on Tool Use with Qwen3-8B.}
    \label{fig:fig3}
    \vspace{-0.1\baselineskip}
\end{wrapfigure}

Despite their effectiveness, existing self-evolution paradigms share a key limitation: their training signals are rarely adapted to the model's evolving capability. Reinforcement learning methods such as GRPO~\citep{shao2024deepseekmath} encourage policy exploration through relative rewards, but the resulting updates can become noisy when successful samples are sparse or reward distributions are unstable. Self-distillation methods such as SDPO~\citep{hubotter2026reinforcement} reuse the model's correct solutions as supervision, yet they typically treat all such solutions uniformly, ignoring whether a problem is easy, near the model's capability boundary, or only occasionally solved. Recent sample-routing methods try to combine self-distillation with reinforcement learning~\citep{li2026srpo}, but most rely on immediate outcomes within the current batch and lack a persistent characterization of problem difficulty and learning progress \cite{zhang2025grpolead}. As a result, training signals are allocated suboptimally: already-mastered problems keep receiving redundant updates, hard problems fail to yield reliable supervision, and high-value boundary cases are under-explored~\cite{zhang2025grpolead,baroian2026prompt,hubotter2026reinforcement}. Compounding this, many methods neither retain nor reuse high-quality correct solutions across batches, preventing the formation of a durable memory for self-evolution~\citep{lin1992self,schaul2016prioritized,zhan2026exgrpo}.

To address these issues, we propose DRIFT, an online self-evolution policy optimization framework for LLMs. DRIFT adaptively coordinates self-distillation and reinforcement learning according to the model's estimated mastery of each problem, balancing error correction, exploration, and stable convergence. Using the model's own correct solutions as references, DRIFT routes each problem to a strategy matched to its difficulty: hard problems receive explicit distillation guidance, problems near the capability boundary drive reinforcement-learning exploration, and reliably solved problems are given fewer updates to avoid unnecessary policy perturbation. Crucially, a cross-batch memory retains high-quality correct solutions as distillation references, enabling hard problems rarely solved in the current batch to benefit from successful trajectories accumulated over time. DRIFT further refines reinforcement-learning updates through token-level modulation, using privileged teacher--student uncertainty dynamics to selectively amplify successful innovations where the student's reasoning decisions remain unstable. By integrating difficulty-aware signal allocation, token-level credit shaping, and cross-batch reuse of successful experience, DRIFT sustains a robust online self-evolution process without any external expert supervision.

We evaluate DRIFT by training Qwen3-8B~\citep{qwen3} on a diverse suite of tasks spanning biology, chemistry, materials, physics, and tool use. DRIFT clearly surpasses our reproduced SDPO and GRPO baselines and further improves over recent methods such as SRPO~\citep{li2026srpo}, with especially pronounced gains on the tool-use task (Figure~\ref{fig:fig3}).

Our contributions are as follows:
\begin{enumerate}
    \item \textbf{Dynamic Difficulty Routing.} Historical pass rates are used as a stable estimate of the model's mastery level, enabling dynamic decisions on when to distill from past successful trajectories, when to explore new strategies, and when to reduce redundant training.
    \item \textbf{Rhythm Gating Exploration.} A structure-preserving exploration mechanism is introduced to encourage policy variation near the success boundary, while preserving useful reasoning patterns through token-level Rhythm Gating.
    \item \textbf{Two-Stage Curriculum Learning.} A two-stage curriculum strategy is designed to align training behaviors with different phases of model learning. The first stage focuses on rapid capability growth and success buffer accumulation, allowing the model to reach a stable performance plateau efficiently. The second stage shifts toward balanced exploration, correction, and convergence, enabling continuous and stable performance improvement.
    \item \textbf{Experience Replay via Success Buffer.} DRIFT incorporates a historical success buffer to reuse high-quality trajectories. It further adopts a curriculum learning strategy to support smooth model evolution and uses Rhythm Gating to provide structured exploration rewards on boundary problems.
\end{enumerate}

\section{Preliminaries}

\subsection{GRPO and SDPO}

Group Relative Policy Optimization (GRPO) is a critic-free policy optimization method commonly used in post-training with verifiable rewards. Given a prompt $x$, the current policy samples a group of $n$ candidate rollouts. Each rollout receives a scalar reward, which is normalized within the group to compute its relative advantage $A_i$. The policy is then optimized using a PPO-style clipped surrogate objective, where $\rho_{i,t}(\theta)=\frac{\pi_\theta(y_{i,t}| x,y_{i,<t})}{\pi_{\theta_{\mathrm{old}}}(y_{i,t}| x,y_{i,<t})}$ denotes the token-level importance ratio:
\[
\mathcal{J}_{\mathrm{GRPO}} (\theta)
=\mathbb{E}_{i,t}\Big[\min\Big(\rho_{i,t}(\theta)\,A_i,\operatorname{clip}\big(\rho_{i,t}(\theta),1\pm\epsilon\big)A_i\Big)\Big].
\]
Since $A_i^{\mathrm{GRPO}}$ is sequence-level, all tokens in the same rollout share the same advantage. GRPO therefore reinforces high-reward trajectories and suppresses low-reward ones, but provides only coarse credit assignment: it cannot identify which tokens are responsible for the final outcome.

Self-Distillation Policy Optimization (SDPO) complements this scalar reward with dense token-level supervision from a self-teacher. Given the student policy $\pi_\theta(\cdot\mid x)$, SDPO constructs a teacher distribution conditioned on auxiliary feedback $f$, written as $\pi_\theta(\cdot\mid x,f)$, where $f$ can be a successful sibling rollout, verifier feedback, execution trace, or other environmental signal. For a sampled trajectory $y_i$, SDPO compares the student and teacher distributions at each prefix position. We adopt the Jensen--Shannon divergence (JSD) for its symmetric treatment of the evolving teacher and student policies and its ability to preserve exploration while stabilizing optimization; see Appendix~\ref{app:divergence} for a detailed analysis. The resulting distillation loss is:
\[
\mathcal{L}_{\mathrm{SDPO}}(\theta)
=
\sum_t
\mathrm{JSD}
\Big(
\pi_\theta(\cdot \mid x, y_{i,<t})
\big\|
\operatorname{sg}
[
\pi_\theta(\cdot \mid x, f, y_{i,<t})
]
\Big).
\]

\begin{figure*}[t]
    \centering
    \includegraphics[width=0.9\linewidth, keepaspectratio]{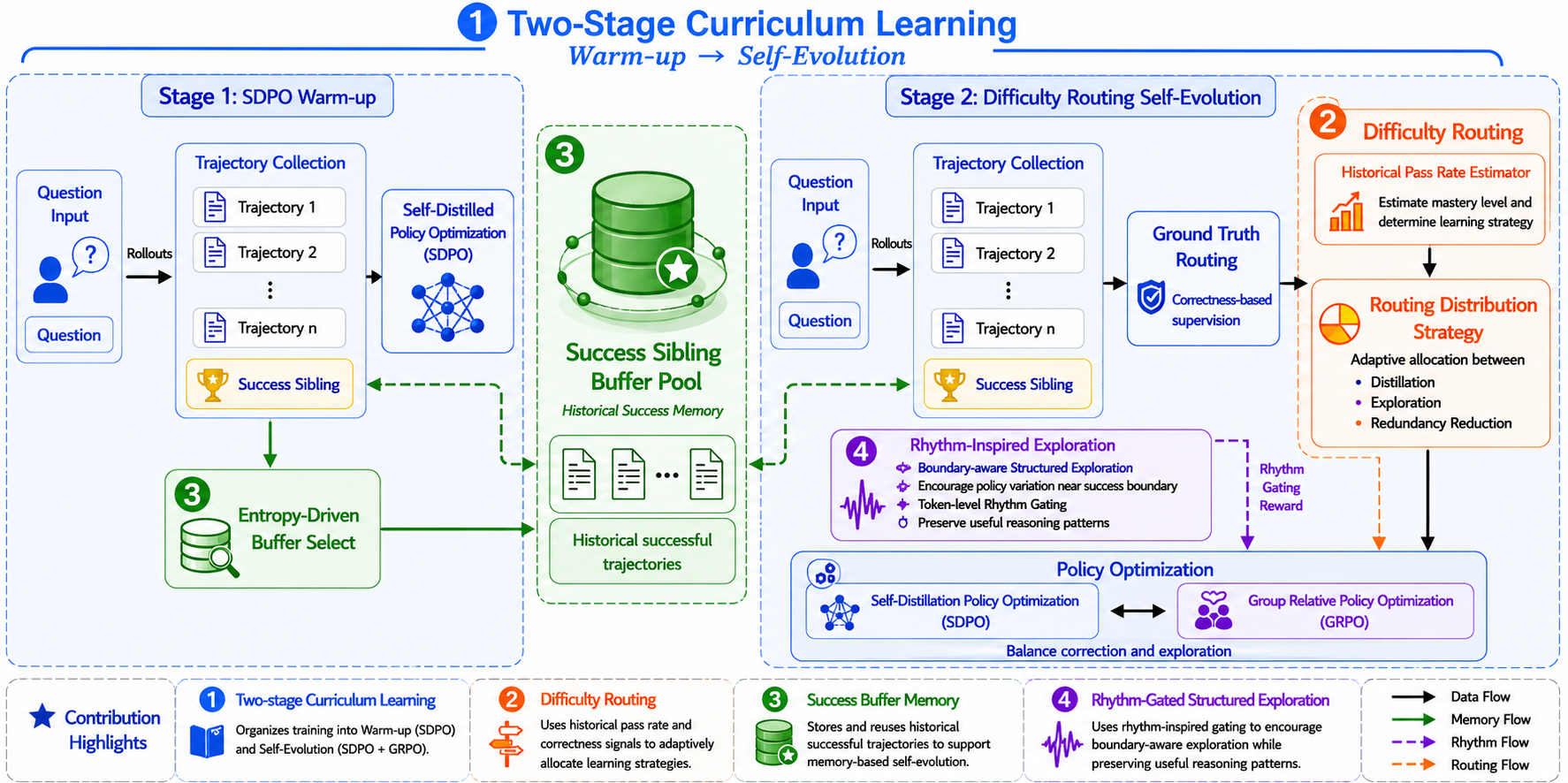}
    \caption{Overall DRIFT training pipeline.}
    \label{fig:drift_pipeline}
\end{figure*}

Although GRPO and SDPO support self-improvement through reinforcement learning and self-distillation, respectively, both lack adaptation to the model's evolving capability. GRPO becomes ineffective when intra-group rewards are homogeneous, and its sequence-level advantage spreads updates over all tokens, preventing fine-grained localization of the reasoning steps that determine correctness. SDPO, in contrast, relies on successful trajectories as supervision but usually treats all successes uniformly, without distinguishing easy problems, boundary cases, or accidental solutions to hard problems. This can lead to redundant imitation on already-mastered samples, insufficient reinforcement of high-value boundary cases, and no effective supervision for hard problems without positive trajectories. Thus, both methods lack a mechanism for tracking problem-level learning state and dynamically allocating optimization signals, which is the core gap addressed by DRIFT.

\section{Method}
DRIFT is built around a two-level signal refinement mechanism and consists of three interlocking components: a \textbf{two-stage curriculum} that first uses self-distillation for a fast warm-up and to accumulate experience, and then switches to a stable optimization stage that combines self-distillation with reinforcement learning; at the problem level, \textbf{difficulty routing} modulates the strength of reinforcement-learning updates according to the model's evolving mastery of each problem, while at the token level, \textbf{rhythm-gated structured exploration} selectively amplifies verified innovations at unresolved decision boundaries; and a privileged self-teacher that grows stronger throughout training, maintained jointly by the teacher-selection rule in the warm-up stage and a cross-stage \textbf{success buffer}, ensuring that dense token-level supervision remains available even when positive samples are scarce. We elaborate on each component below. The overall pipeline and the complete training procedure are summarized in Figure~\ref{fig:drift_pipeline} and Algorithm~\ref{alg:drift}, respectively.

\subsection{Curriculum Teacher Construction}
\label{sec:curriculum}

DRIFT organizes training into a warm-up stage and a mixed-optimization stage that share a single self-teacher, played by the model itself, whose supervision is provided as privileged information (a success sibling, denoted $\mathrm{sib}^{+}$).

\begin{algorithm}[t]
\caption{DRIFT Training}
\label{alg:drift}
\small
\setlength{\baselineskip}{0.92\baselineskip}
\setlength{\itemsep}{0pt}
\begin{algorithmic}[1]
\Require Dataset $\mathcal{D}$, policy $\pi_\theta$, rollout size $n$, reward threshold $\theta_r$, warm-up steps $T_{\mathrm{warm}}$, buffer epoch $k$, capacity $b$
\Ensure Optimized policy $\pi_\theta$
\State Initialize pass-rate table $\{\bar{r}(x)\}$ and success buffers $\mathcal{B}(x)\leftarrow\emptyset$ for all $x\in\mathcal{D}$
\For{training step $t=1,\ldots,T$}
    \State Sample a batch of prompts $\{x\}$ from $\mathcal{D}$
    \State Collect rollouts $\{y_i\}$ and rewards $\{r_i\}$ by sampling $n$ responses per prompt
    \State Compute GRPO advantages $\{A_i^{\mathrm{GRPO}}\}$ jointly over all rollouts in the batch
    \State $\mathcal{L}\leftarrow 0$
    \For{each prompt $x$ in the batch}
        \State Let $\mathcal{Y}(x)=\{y_i\}_{i\in\mathcal{I}(x)}$ be its rollout group with rewards $\{r_i\}_{i\in\mathcal{I}(x)}$
        \State Update $p(x)$ by EMA of the group pass rate \Comment{Sec.~\ref{sec:routing}}
        \State Set routing weights $\gamma$ from $p(x)$
        \State Select $\mathrm{sib}^{+}$ from $\mathcal{Y}(x)$, or from $\mathcal{B}(x)$ if no positive rollout and $\mathrm{epoch}>k$ \Comment{Secs.~\ref{sec:curriculum},~\ref{sec:buffer}}
        \State Update $\mathcal{B}(x)$ with successful rollouts by reward-based replacement \Comment{Sec.~\ref{sec:buffer}}
        \For{each $i\in\mathcal{I}(x)$}
            \If{$t\le T_{\mathrm{warm}}$} \Comment{Stage 1: SDPO warm-up}
                \State $\mathcal{L}\leftarrow\mathcal{L}+\mathcal{L}_{\mathrm{SDPO}}(y_i,\mathrm{sib}^{+})$
            \ElsIf{$r_i<\theta_r$} \Comment{Stage 2: incorrect $\rightarrow$ SDPO}
                \State $\mathcal{L}\leftarrow\mathcal{L}+\mathcal{L}_{\mathrm{SDPO}}(y_i,\mathrm{sib}^{+})$
            \Else \Comment{Stage 2: correct $\rightarrow$ rhythm-GRPO}
                \State Compute token weights $\{M^{\mathrm{rhy}}_{i,t}\}$ from teacher--student log-probs and entropy dynamics \Comment{Sec.~\ref{sec:rhythm}}
                \State $\mathcal{L}\leftarrow\mathcal{L}-\gamma\, \mathcal{J}^{\mathrm{rhythm}}_{\mathrm{GRPO}}(y_i;\,A_i^{\mathrm{GRPO}},\,M^{\mathrm{rhy}}_{i})$
            \EndIf
        \EndFor
    \EndFor
    \State Update $\theta$ by minimizing $\mathcal{L}$
\EndFor
\end{algorithmic}
\end{algorithm}

\paragraph{Stage 1: Entropy-based teacher selection for the SDPO warm-up.}
The warm-up stage is SDPO-dominated and differs from standard SDPO only in how the success sibling is chosen. Rather than taking the first sample with $\text{reward}\ge\text{threshold}$, we use a ``reward-first, entropy-second'' rule: among samples whose reward lies in $[\,r_{\max}-\delta,\;r_{\max}\,]$, we select the highest-entropy one ($\delta$ a tolerance). A high-entropy teacher preserves exploration and diversity early in training, while this stage simultaneously accumulates successful trajectories for the success buffer (Section~\ref{sec:buffer}).

\paragraph{Stage 2: Mixed optimization of SDPO and GRPO.}
In the mixed stage we route each sample by correctness: $\text{reward}<\theta_r$ is negative, otherwise positive. \textbf{Negative samples take the SDPO branch}, with $\mathrm{sib}^{+}$ as teacher---now chosen by largest reward, ties broken by shortest response length---while \textbf{positive samples take the rhythm-GRPO branch}. 

Positive and negative samples exhibit a natural \textbf{asymmetry} under both policy-gradient and distillation updates. At position $t$, let $p=\pi_\theta(\cdot\mid x,y_{i,<t})$ be the student's next-token distribution, $a$ the sampled token, and $q=\pi_\theta(\cdot\mid x,f,y_{i,<t})$ the stop-gradient teacher; since $\nabla_{z}\log p_a=e_a-p$ (with $e_a$ the one-hot of $a$), the unclipped GRPO update in logit space is $\Delta z\propto A_i(e_a-p)$. On a negative sample ($A_i<0$), this update only removes probability from $a$ and redistributes it over the remaining tokens in proportion to $p$, independently of the correct token, so the learning signal is poorly utilized; we therefore route negatives to SDPO, which descends $\mathrm{JSD}(p\|q)$ toward the privileged (correct) teacher $\mathrm{sib}^{+}$, providing an explicit corrective target that GRPO lacks. On a positive sample, distilling from a distinct correct sibling ($p\neq q$) pulls the policy toward a predetermined correct solution, shrinking solution diversity and injecting a distribution-misaligned gradient that destabilizes training; by contrast, the GRPO update $A_i(e_a-p)$ with $A_i>0$ raises $p_a$ for the sample's own correct token and thereby reinforces the verified correct answer. We provide a fuller analysis of this branch assignment in Appendix~\ref{sec:app-branch-geometry}.

Writing the two branches together, we obtain the overall objective of DRIFT in the mixed stage:
\begin{equation*}
\label{eq:master}
\min_{\theta}\;\;\mathcal{L}(\theta):=
\mathbb{E}_{y_i\sim\pi_{\theta}(\cdot\mid x)}
\left[
\underbrace{
{\color{cIncorrect}\mathbf{1}_i^{-}}\,
\mathcal{L}_{\mathrm{SDPO}}(y_i,f_i)
}_{\text{Self-Distillation}\;(\S\ref{sec:curriculum},\,\S\ref{sec:buffer})}
-
\underbrace{
{\color{cGamma}\gamma_i}\,
\mathbf{1}_i^{+}\,
\mathcal{J}_{\mathrm{GRPO}}^{\mathrm{rhythm}}(y_i)
}_{\text{Difficulty-Routed RL}\;(\S\ref{sec:routing},\,\S\ref{sec:rhythm})}
\right],
\end{equation*}
where the two branches expand respectively as
\begin{align*}
\mathcal{L}_{\mathrm{SDPO}}(y_i,f_i)&:=\sum_{t=1}^{|y_i|}
\mathrm{JSD}\!\left(\pi_\theta(\cdot\mid x,y_{i,<t})\,\big\Vert\,
\sg\!\left[\pi_\theta(\cdot\mid x,f_i,y_{i,<t})\right]\right), \\[4pt]
\mathcal{J}_{\mathrm{GRPO}}^{\mathrm{rhythm}}(y_i)&:=\sum_{t=1}^{|y_i|}
{\color{cRhythm}M^{\mathrm{rhy}}_{t}}\cdot
\min\!\Big(\rho_{i,t}(\theta)\,A_i,\;
\mathrm{clip}\!\big(\rho_{i,t}(\theta),1\pm\epsilon\big)\,A_i\Big).
\end{align*}

This objective uses three colored terms to mark the division of labor among the following subsections: ${\color{cIncorrect}\mathbf{1}_i^{-}}$ triggers the SDPO privileged-correction branch, whose teacher $f_i$ is provided jointly by the teacher-selection rule of this section and the success buffer of Section~\ref{sec:buffer}; ${\color{cGamma}\gamma_i}$ is the problem-level difficulty-routing weight, given by Section~\ref{sec:routing}; and ${\color{cRhythm}M^{\mathrm{rhy}}_{t}}$ is the token-level rhythm-gating weight, given by Section~\ref{sec:rhythm}. The following three subsections refine these three quantities in turn.

\subsection{Difficulty Routing}
\label{sec:routing}

This section refines problem-level signal allocation through ${\color{cGamma}\gamma_i}$ in the objective. Sample difficulty affects policy optimization non-monotonically: overly hard samples rarely yield positive, generalizable trajectories, while overly easy samples produce homogeneous rewards that mainly reinforce surface-level patterns~\citep{cheng2026sampledifficulty,baroian2026prompt,chen2026unlearnability}. We further observe that on multiple-choice datasets, models can guess correct answers via flawed reasoning, further miscalibrating the signal from easy samples. Medium-difficulty samples therefore benefit most from GRPO. Accordingly, in the second curriculum stage, we down-weight positive GRPO samples that are overly easy or overly hard, keeping full strength only for medium-difficulty cases; this weight is exactly ${\gamma_i}$.

\paragraph{Pass-rate estimation.}
We use each problem's pass rate to estimate difficulty. Given $n$ rollouts, the instantaneous pass rate is
\begin{equation*}
p_{\mathrm{now}} = \frac{n_{\mathrm{correct}}}{n}.
\end{equation*}
Since $n$ is usually small (8 or 16), this estimate can fluctuate substantially across epochs even for a fixed model. We therefore smooth it using an exponential moving average of the historical pass rate:
\begin{equation*}
p = \alpha\, p_{\mathrm{past}} + (1-\alpha)\, p_{\mathrm{now}}.
\end{equation*}

\paragraph{Difficulty bins and weights.}
Based on the smoothed pass rate $p$, we assign
\begin{equation*}
{gamma_i} =
\begin{cases}
\gamma_{\mathrm{easy}}, & p > p_{\mathrm{easy}}, \\
1, & p_{\mathrm{hard}} < p < p_{\mathrm{easy}}, \\
\gamma_{\mathrm{hard}}, & p < p_{\mathrm{hard}}.
\end{cases}
\end{equation*}
The thresholds $p_{\mathrm{hard}}$ and $p_{\mathrm{easy}}$ are not chosen heuristically, but are instead derived from the pass rate at which a binary-reward group's gradient signal vanishes.

\begin{lemma}[GRPO group degeneration]
\label{lem:grpo-deadzone}
Assume $n$ i.i.d.\ binary rewards with pass rate $p$. The probability that a group is all-correct or all-incorrect, and hence yields zero GRPO gradient, is
\begin{equation*}
P_{\mathrm{deg}}(p)=(1-p)^n+p^n .
\end{equation*}
Defining the informative band $\{p:P_{\mathrm{deg}}(p)<\tau\}$ and solving the hard-side boundary $(1-p)^n=\tau$ gives
\begin{equation*}
p_{\mathrm{hard}}=1-\tau^{1/n}\approx\frac{-\ln\tau}{n},\qquad p_{\mathrm{easy}}=1-p_{\mathrm{hard}},
\end{equation*}
where the approximation follows from the large-$n$/small-$p$ expansion $e^{-np}=\tau$.
\end{lemma}

The proof and the reason of selecting thresholds are deferred to Appendix~\ref{sec:app-grpo-deadzone}.

\subsection{Rhythm-Gated Structured Exploration}
\label{sec:rhythm}

This section defines the token weight ${\color{cRhythm}M^{\mathrm{rhy}}_{t}}$ in the objective.
Standard GRPO applies one sequence-level advantage to every token, so low-semantic positions receive the same exploration signal as decision-critical ones.
On the medium-difficulty GRPO branch we therefore apply a \textbf{rhythm-gated rebellious bonus}: amplify a token only when a verified student trajectory exhibits a teacher-relative innovation at a position where the privileged teacher resolves uncertainty faster than the student.

\begin{figure}[t]
    \centering
    \includefig[width=\linewidth]{Figures/visual_physics.png}
    \caption{Rhythm-gating visualization on a correct Physics rollout
    (Qwen3-8B trained with DRIFT on Physics).
    {Red}: tokens kept by the gate,
    intensity $\propto b_t^{\mathrm{reb}}g_t^{\mathrm{rhythm}}$.
    {Blue}: rebellious exceedances suppressed by the gate,
    $\propto b_t^{\mathrm{reb}}$.
    }
    \label{fig:rhythm-vis-physics}
\end{figure}

Figure~\ref{fig:rhythm-vis-physics} previews this effect on a Physics rollout: blue marks rebellious exceedances on connectives and formatting that the gate suppresses, while only the red positions---true reasoning turning points---are retained as structurally critical tokens.
Additional Tool-Use visualizations are given in Appendix~\ref{sec:app-rhythm-vis}.
We construct the two factors below.

\paragraph{Reward-validated innovation: teacher-relative rebellious bonus.}
Let $\ell_t^S$ and $\ell_t^T$ denote the log-probabilities assigned by the student and the privileged teacher, respectively, to the actually sampled token at response position $t$. Their difference is the sampled-token log-likelihood ratio
\[
\ell_t^S-\ell_t^T
=
\log\frac{\pi_S(y_t\mid x,y_{<t})}
{\pi_T(y_t\mid x,f,y_{<t})}.
\]
Writing $(x)_{+}=\max(x,0)$, we define the bounded rebellious bonus
\begin{equation*}
b_t^{\mathrm{reb}}=\tanh\!\big((\ell_t^S-\ell_t^T)_{+}\big)\in[0,1).
\end{equation*}

The quantity $\ell_t^S>\ell_t^T$ alone does not imply that the student is globally better than the teacher; it only identifies a token locally preferred by the student policy. Crucially, however, the bonus is used only on the positive GRPO branch, where the entire rollout has passed the verifier. The outcome reward therefore supplies posterior validation of this local deviation. We consequently interpret $b_t^{\mathrm{reb}}$ as a \emph{reward-validated policy innovation}: it measures how strongly a successful student trajectory departs from its privileged reference, without treating the teacher as an insurmountable upper bound. The $\tanh$ mapping makes this innovation score monotone and bounded without an additional clipping hyperparameter.

\paragraph{Uncertainty-resolution lag: rhythm gate.}
A successful innovation is not necessarily important: higher student likelihood may occur on punctuation, connectives, formatting tokens, or repeated phrases. We therefore gate it using the temporal structure of predictive uncertainty. Let $H_t^S$ and $H_t^T$ be the next-token entropies of the student and privileged teacher at position $t$. 
For a local window of length $W$, define
\[
\bar H^{*}_{\mathrm{past}}(t)=\frac{1}{W}\sum_{j=t-W}^{t-1}H_j^{*},
\qquad
\bar H^{*}_{\mathrm{future}}(t)=\frac{1}{W}\sum_{j=t}^{t+W-1}H_j^{*},
\]
for $*\in\{S,T\}$, and 
\begin{equation*}
\mathrm{Drop}_*(t)=\big(\bar H^{*}_{\mathrm{past}}(t)-\bar H^{*}_{\mathrm{future}}(t)\big)_+.
\end{equation*}

\begin{lemma}[Local-slope interpretation and offset invariance]
\label{lem:rhythm-slope}
Fix a time $t$, and let $\mathcal{N}(t)$ be the index set spanned by the two
windows that define $\mathrm{Drop}_H(t)$ (of total width $2W$). Suppose the
entropy trajectory is \emph{locally affine at $t$}: there is a neighborhood of
$t$ containing $\mathcal{N}(t)$, a local slope $c=c(t)$, and an offset $a=a(t)$
such that
\[
H_j = a + c\,j \qquad \text{for all } j \in \mathcal{N}(t).
\]
Then
\[
\mathrm{Drop}_H(t) = W\cdot (-c)_+ .
\]
Moreover, for any constant confidence offset $d$,
$\mathrm{Drop}_{H+d}(t) = \mathrm{Drop}_H(t)$.
\end{lemma}
The proof is deferred to Appendix~\ref{sec:app-rhythm-slope}.

Thus $\mathrm{Drop}$ estimates the positive local uncertainty-resolution rate rather than absolute entropy level, and is invariant to a static confidence gap induced by privileged context.
We define the student's \emph{uncertainty-resolution lag} and the rhythm gate as
\begin{equation*}
R_t=\big(\mathrm{Drop}_T(t)-\mathrm{Drop}_S(t)\big)_{+}, \quad
g_t^{\mathrm{rhythm}}=\tanh(R_t)\in[0,1).
\end{equation*}
We empirically set the window size to $W=10$ for all experiments.

\paragraph{Exploration weight.}
Combining the two factors, the token weight actually used on the GRPO branch is
\begin{equation*}
\label{eq:rhy-weight}
{\color{cRhythm}M^{\mathrm{rhy}}_{i,t}}
=
1+\beta_i \,
\underbrace{b_{i,t}^{\mathrm{reb}}}_{\substack{\text{Reward-validated}\\\text{innovation}}}
\underbrace{g_{i,t}^{\mathrm{rhythm}}}_{\substack{\text{Uncertainty-resolution}\\\text{lag}}},
\end{equation*}
where $\beta_i=1$ if and only if sample $i$ lies in the medium-difficulty region of Section~\ref{sec:routing}, and $\beta_i=0$ otherwise. The product implements a logical conjunction: innovation without resolution lag may be a surface-level deviation, while resolution lag without innovation supplies no new successful behavior to reinforce. Additional credit is assigned only when a verified student innovation occurs at a privileged decision boundary on a problem near the capability frontier.

\begin{remark}[bounded amplification]
    Since $\beta_i\in\{0,1\}$ and $b_{i,t}^{\mathrm{reb}},g_{i,t}^{\mathrm{rhythm}}\in[0,1)$, we have $1\le {M^{\mathrm{rhy}}_{i,t}}<2$. Thus, on a positive-advantage rollout, rhythm gating preserves the update direction and only adds a bounded emphasis to verified innovations. Difficulty routing selects {which problems} lie near the capability frontier, while rhythm gating identifies {where} their successful but unstable decisions emerge.
\end{remark}

\subsection{Success Buffer}
\label{sec:buffer}

Finally, we describe how to keep the privileged self-teacher for the ${\color{cIncorrect}\mathbf{1}_i^{-}}$ SDPO branch consistently available and progressively stronger, especially when positive samples are scarce. Inspired by RLEP~\citep{zhang2025rlep}, we introduce a success buffer into OPD. The key idea is to reuse rollouts that have received positive verification signals, so that the model can continually reinforce effective reasoning paths rather than relying on fresh exploration at every update. During sampling, successful rollouts are stored in a per-sample buffer and later replayed as $\mathrm{sib}^{+}$ when the current rollout group contains no positive sample. Different from a static offline replay buffer, our buffer is dynamically updated: as the policy improves, higher-reward successful trajectories replace lower-quality historical ones, providing an implicit privileged signal that strengthens throughout training.

For each training sample $x$, we maintain a small success buffer
\[
\mathcal{B}(x)=\{(y_j^+, r_j)\}_{j=1}^{m_x},
\qquad
m_x \le b,
\]
where $y_j^+$ is a historical successful rollout and $r_j$ is its reward. The buffer is only populated during the first $k$ epochs. Starting from epoch $k+1$, if the current rollout group
\[
\mathcal{Y}(x)=\{y_1,\ldots,y_n\}
\]
contains no successful sample, i.e.,
\[
\max_{y \in \mathcal{Y}(x)} r(y) \le 0,
\]
we randomly draw one entry from $\mathcal{B}(x)$ and use it as $\mathrm{sib}^{+}$.

The update rule is simple. If the buffer is not full, a new successful rollout is appended directly; once it contains $b$ entries, each new success evicts the earliest inserted one. We set $b=3$ and uniformly sample one stored sibling whenever replay is needed.

A natural concern is that historical successes in the buffer may exhibit distribution shift relative to trajectories induced by the current policy. Unlike importance-sampling-based methods, however, buffered trajectories in DRIFT neither enter the policy-gradient estimator as off-policy samples nor directly reweight gradient updates through likelihood ratios. Instead, a retrieved success is supplied only as privileged context to the stop-gradient self-teacher, which then produces a conditional token-level distribution over the current student's on-policy prefix. The buffer's influence is therefore mediated by the teacher's contextual inference and subsequently translated into dense token-level supervision, rather than being propagated as an unfiltered trajectory-level weight. This mediation acts as an implicit soft filter: historical information affects optimization only insofar as the current teacher can interpret and adapt it to the present context.

\section{Experiments}

\subsection{Experimental Setup}

\paragraph{Models and training.}
We conduct experiments on three instruction-tuned models, including Qwen3-8B~\citep{qwen3}, Qwen3-4B, and OLMo3-7B~\citep{olmo2025olmo3}, covering both different scales within the same model family and models from different families. Qwen3-8B serves as our primary model; unless otherwise noted, all ablation studies and training-dynamics analyses beyond the main performance comparison are conducted on Qwen3-8B. All experiments are trained for more than400 steps on 8 NVIDIA H200 GPUs. 
Unlike prior work that reports results in terms of wall-clock time~\citep{hubotter2026reinforcement,li2026srpo}, we follow SC-SDPO~\citep{liu2026scsdpo} and report results in terms of training steps, since optimization steps are independent of hardware configuration and enable reproducible comparison across different compute environments.

\paragraph{Datasets.}
Following the evaluation setting of SDPO, we conduct experiments on five benchmarks: Chemistry, Physics, Biology, Materials, and Tool Use. The first four benchmarks are drawn from the reasoning subsets of SciKnowEval~\citep{feng2024sciknoweval}, covering undergraduate-level scientific reasoning across multiple domains. The Tool Use benchmark assesses the model's ability to convert a user query and a given tool specification into the correct tool invocation, following ToolAlpaca~\citep{tang2023toolalpaca}. Following SDPO, we split each benchmark into training and test sets to measure in-domain generalization.

\paragraph{Baselines.}
We compare DRIFT with six baselines: GRPO~\citep{shao2024deepseekmath}, SDPO~\citep{hubotter2026reinforcement}, SRPO~\citep{li2026srpo}, DistIL~\citep{agrawal2026distil}, SC-SDPO~\citep{liu2026scsdpo}, and PGPO~\citep{wang2026pgpo}. We reproduce GRPO and SDPO using the implementation framework provided by SDPO~\citep{hubotter2026reinforcement}, with the default hyperparameter settings kept unchanged. For SRPO, DistIL, SC-SDPO, and PGPO, we use the results reported in their corresponding papers~\citep{li2026srpo, agrawal2026distil, liu2026scsdpo, wang2026pgpo}. When comparing with reported numbers, we follow the same benchmark and evaluation protocol whenever applicable.

\subsection{Main Results}

\paragraph{Main results.}
Table~\ref{tab:main_qwen3_8b} reports the mean@16 accuracy of Qwen3-8B across four scientific reasoning tasks and tool use. DRIFT achieves the best overall performance, with an average accuracy of 79.5\%, outperforming the strongest baseline SRPO by 2.1 points and improving over the Qwen3-8B starting point by 30.0 points. More importantly, DRIFT performs consistently across all tasks, ranking within the top two in every category, rather than improving one domain at the cost of another.

This robustness comes from difficulty-aware routing. On hard samples, where group-level rewards in GRPO can collapse into coarse trajectory suppression, DRIFT routes optimization to the SDPO branch and uses $\mathrm{sib}^{+}$ to provide finer token-level correction. On easier samples, difficulty weighting prevents overly superficial reinforcement. As a result, DRIFT maintains a better balance between exploration, correction, and stable policy improvement across heterogeneous tasks.
\begin{table}[htb]
\centering
\small
\renewcommand{\arraystretch}{1.25}
\setlength{\tabcolsep}{10pt}
\begin{tabular}{lcccccc}
\toprule
Method & Biology & Chemistry & Material & Physics & Tool Use & Average \\
\midrule
\textbf{Qwen3-8B} & 30.8 & 41.2 & 58.9 & 59.2 & 57.5 & 49.5 \\
+ SDPO$^{\ast}$ & 56.8 & 80.9 & 78.4 & 75.6 & 68.5 & 72.0 \\
+ SDPO & 64.8 & 78.9 & 76.1 & 72.7 & 67.7 & 72.0 \\
+ GRPO$^{\ast}$ & 59.9 & 74.5 & 77.1 & 72.7 & 65.7 & 70.0 \\
+ GRPO & 47.4 & 65.6 & 73.5 & 60.6 & 67.7 & 63.0 \\
+ SRPO$^{\ast}$ & \underline{72.8} & \textbf{83.0} & \textbf{81.5} & 78.4 & \underline{71.2} & \underline{77.4} \\
+ DistIL$^{\ast}$ & 66.6 & 80.8 & 76.2 & \underline{80.8} & -- & -- \\
+ SC-SDPO$^{\ast}$ & 65.4 & 80.6 & 79.3 & \textbf{81.6} & 67.3 & 74.8 \\
+ PGPO$^{\ast}$ & 61.3 & 77.6 & 78.7 & 77.6 & -- & -- \\
\midrule
\textbf{+ DRIFT (Ours)} & \textbf{74.4}\imp{9.6} & \underline{82.0}\imp{3.1} & \underline{81.4}\imp{5.3} & 80.5\imp{7.8} & \textbf{79.2}\imp{11.5} & \textbf{79.5}\imp{7.5} \\
\bottomrule
\end{tabular}
\caption{Comparison of different post-training methods based on Qwen3-8B, reported as mean@16 accuracy (\%). Results marked with $\ast$ are cited from prior work. \textbf{Bold} and \underline{underlined} values indicate the best and second-best results, respectively. {\color{impgreen}Green subscripts} denote absolute gains over the reproduced SDPO baseline.}
\label{tab:main_qwen3_8b}
\end{table}

DRIFT's effectiveness stems from balancing exploration with structured knowledge learning. The self-distillation branch turns successful siblings into supervision for reusable knowledge such as tool protocols, domain facts, and answer formats, while rhythm-gated exploration on the policy-gradient branch prevents premature convergence to specific trajectories. Together they improve both tool-use and STEM performance. Gains are larger on tool use because correctness is defined by schemas, \texttt{Action}/\texttt{Action Input} formats, and argument constraints, so successful demonstrations transfer cleanly across samples; STEM knowledge is more problem-specific and less reusable, yielding more moderate gains.

\paragraph{Generalization across model families and scales.}
Table~\ref{tab:main_qwen4b_olmo7b} reports results on Qwen3-4B and Olmo3-7B. DRIFT again attains the highest average accuracy (74.6\% and 72.8\%), and its advantage on tool use remains the largest (leading the second-best by 7.0 and 7.3 points, respectively). We highlight two observations. First, the gain amplifies as the starting point weakens: on Olmo3-7B (base 30.5\%), DRIFT improves by 42.3 points and leads online GRPO by 12 points, because when positive samples are scarce, the privileged replay from the success buffer fills exactly the void left by GRPO's advantage collapse. Second, the lead over the strongest baseline SRPO widens with scale (+0.4 at 4B, +2.1 at 8B), consistent with the observation that self-teaching ability emerges with increasing model scale.

\begin{table}[htb]
\centering
\small
\renewcommand{\arraystretch}{1.25}
\setlength{\tabcolsep}{10pt}
\begin{tabular}{lcccccc}
\toprule
Method & Biology & Chemistry & Material & Physics & Tool Use & Average \\
\midrule
\textbf{Qwen3-4B} & 30.8 & 43.6 & 61.2 & 59.8 & 57.9 & 50.7 \\
+ SDPO$^{\ast}$ & 54.0 & 77.3 & 74.3 & 66.7 & 61.1 & 66.7 \\
+ GRPO$^{\ast}$ & 55.5 & 78.3 & 80.1 & 71.9 & 62.9 & 69.7 \\
+ SRPO$^{\ast}$ & \textbf{65.8} & \textbf{82.7} & \underline{81.3} & \textbf{74.0} & \underline{67.0} & \underline{74.2} \\
\textbf{+ DRIFT (Ours)} & \underline{64.3}\imp{10.3} & \underline{79.6}\imp{2.3} & \textbf{82.8}\imp{8.5} & \underline{72.2}\imp{5.5} & \textbf{74.0}\imp{12.9} & \textbf{74.6}\imp{7.9} \\
\midrule
\textbf{Olmo3-7B} & 16.2 & 22.8 & 36.7 & 37.7 & 39.3 & 30.5 \\
+ SDPO$^{\ast}$ & 52.8 & 80.0 & \underline{79.1} & 66.1 & 62.1 & 68.0 \\
+ GRPO$^{\ast}$ & 49.8 & 57.5 & 73.5 & 62.7 & 60.6 & 60.8 \\
+ SC-SDPO$^{\ast}$ & \underline{55.9} & \underline{81.5} & \underline{79.1} & 71.1 & \underline{62.9} & \underline{70.1} \\
+ DistIL$^{\ast}$ & 55.3 & 81.0 & 76.9 & \textbf{74.5} & -- & -- \\
\textbf{+ DRIFT (Ours)} & \textbf{61.3}\imp{8.5} & \textbf{82.0}\imp{2.0} & \textbf{79.2}\imp{0.1} & \underline{71.3}\imp{5.2} & \textbf{70.2}\imp{8.1} & \textbf{72.8}\imp{4.8} \\
\bottomrule
\end{tabular}
\caption{Comparison of different post-training methods on Qwen3-4B and Olmo3-7B, reported as mean@16 accuracy (\%). We use the same notation as in Table~\ref{tab:main_qwen3_8b}.}
\label{tab:main_qwen4b_olmo7b}
\end{table}

\subsection{Training Dynamics}
\label{sec:training_dynamics}

\paragraph{Training dynamics.}
Figures~\ref{fig:training_dynamics} and~\ref{fig:routing_dynamics} summarize the evolution of DRIFT during training. Actor entropy rises during early exploration and stabilizes after approximately 200 steps, while response length remains bounded. The success buffer grows rapidly before approaching saturation, with intermittent fallback hits indicating selective reuse of historical successes. At the same time, the easy-problem fraction gradually increases, the medium group remains dominant, and the hard fraction declines before stabilizing. Together, these trends show that DRIFT progressively shifts from broad exploration toward stable refinement while continuously adapting its routing decisions to the model's evolving capability.
\begin{figure}[htbp]
    \centering
    \begin{subfigure}[t]{0.48\textwidth}
        \centering
        \includegraphics[width=\linewidth]{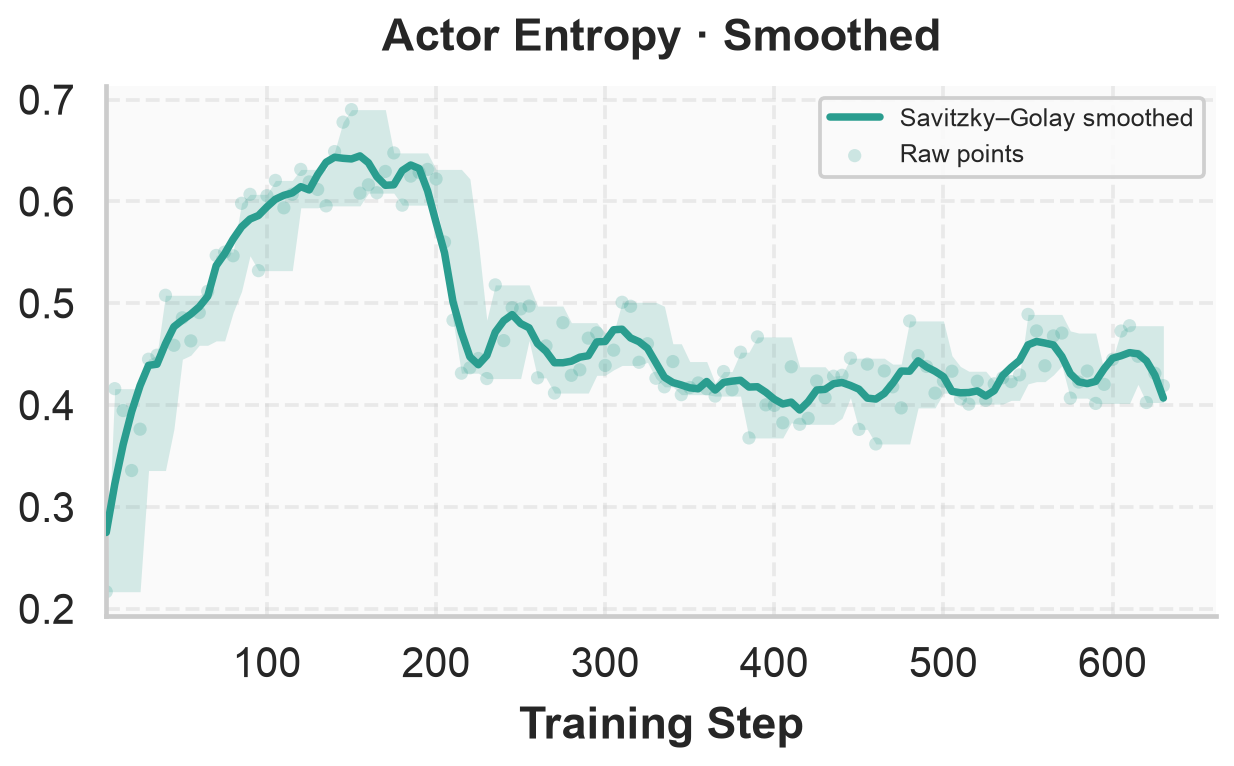}
    \end{subfigure}
    \hfill
    \begin{subfigure}[t]{0.48\textwidth}
        \centering
        \includegraphics[width=\linewidth]{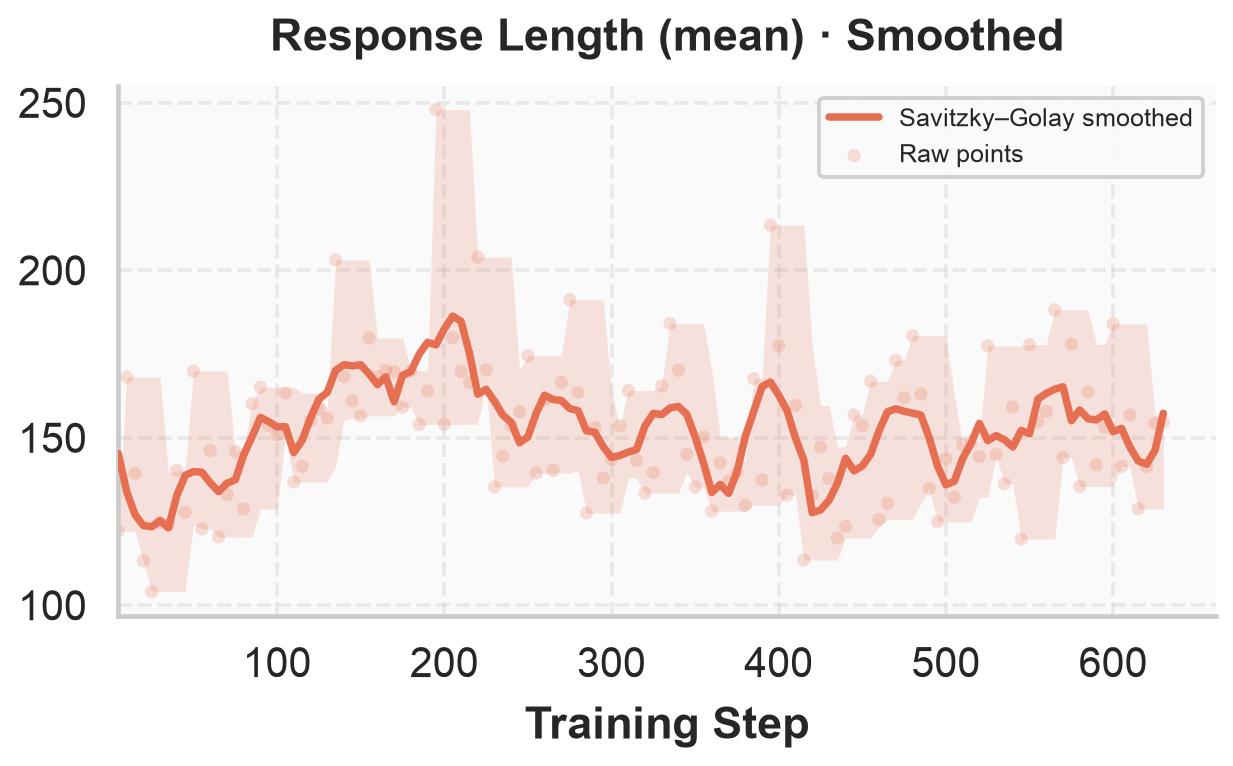}
    \end{subfigure}

    \vspace{0.3cm}

    \begin{subfigure}[t]{0.48\textwidth}
        \centering
        \includegraphics[width=\linewidth]{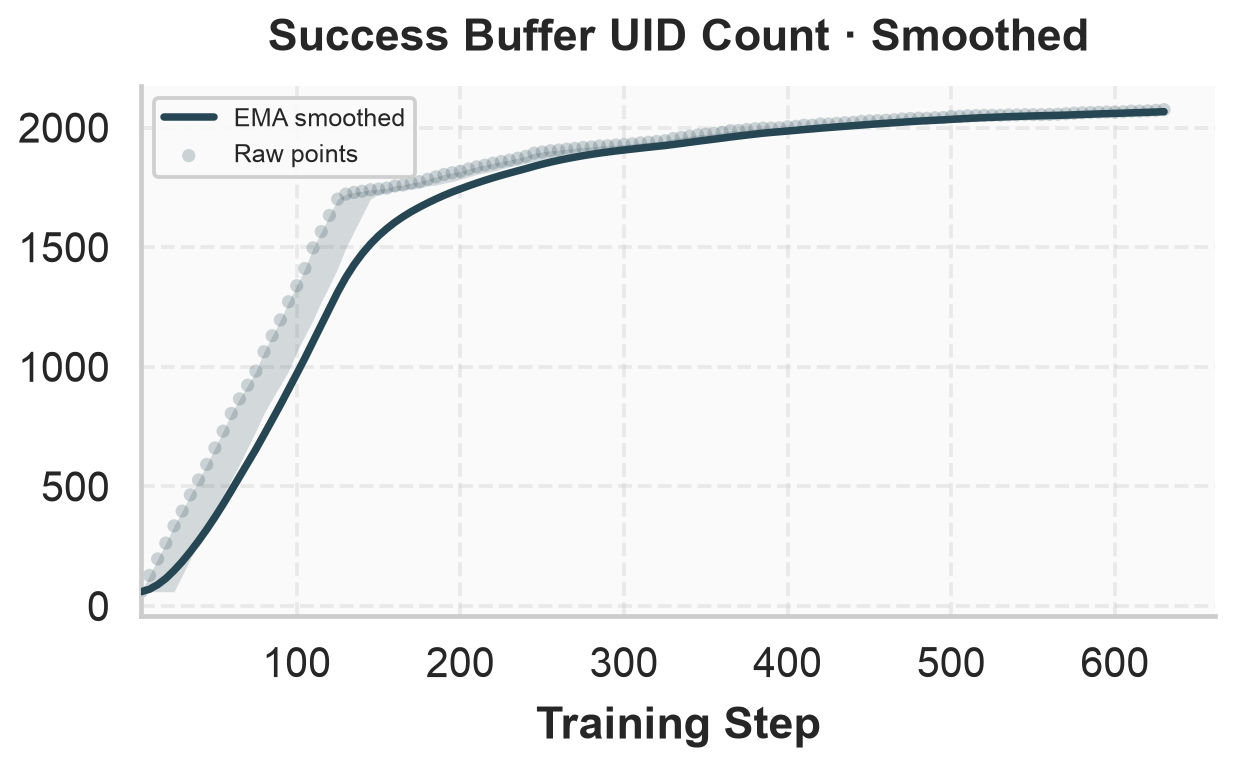}
    \end{subfigure}
    \hfill
    \begin{subfigure}[t]{0.48\textwidth}
        \centering
        \includegraphics[width=\linewidth]{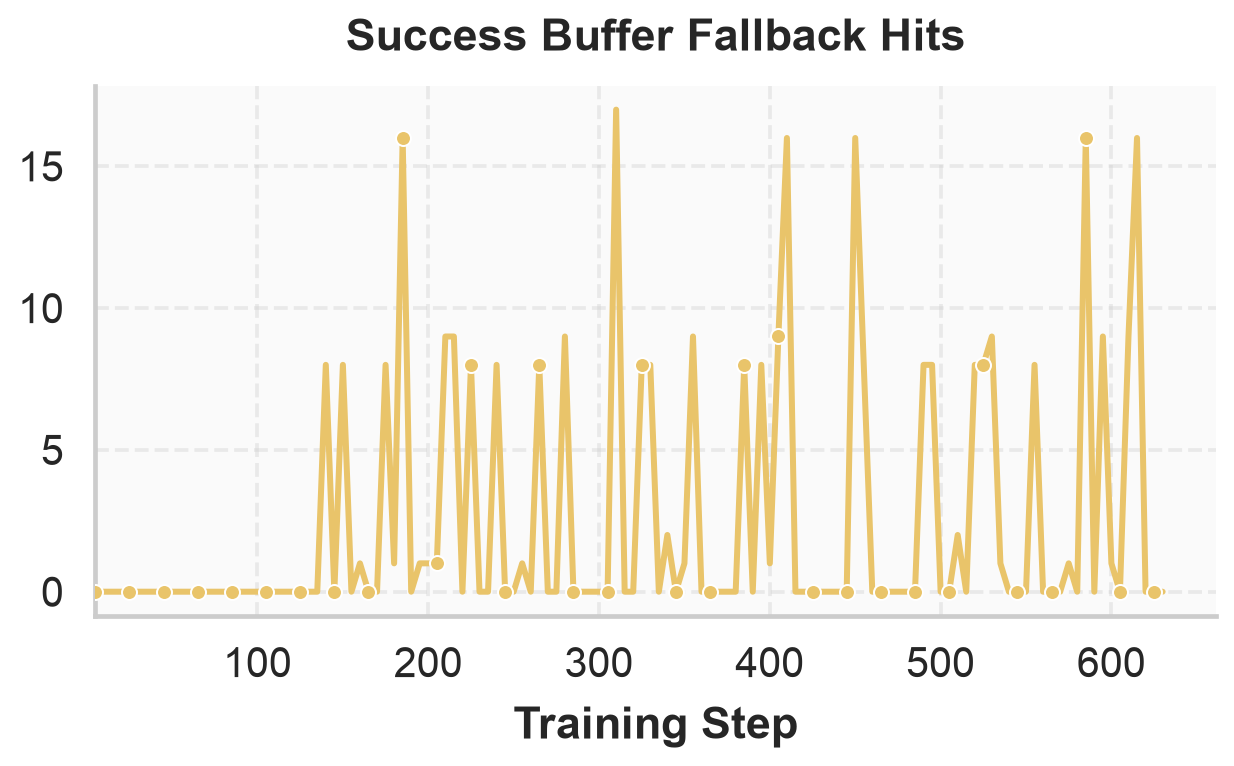}
    \end{subfigure}
    \caption{\textbf{Training dynamics of DRIFT (Qwen3-8B in tooluse)}}
    \label{fig:training_dynamics}
\end{figure}

\begin{figure}[htbp]
    \centering
    \begin{subfigure}[t]{0.32\textwidth}
        \centering
        \includegraphics[width=\linewidth]{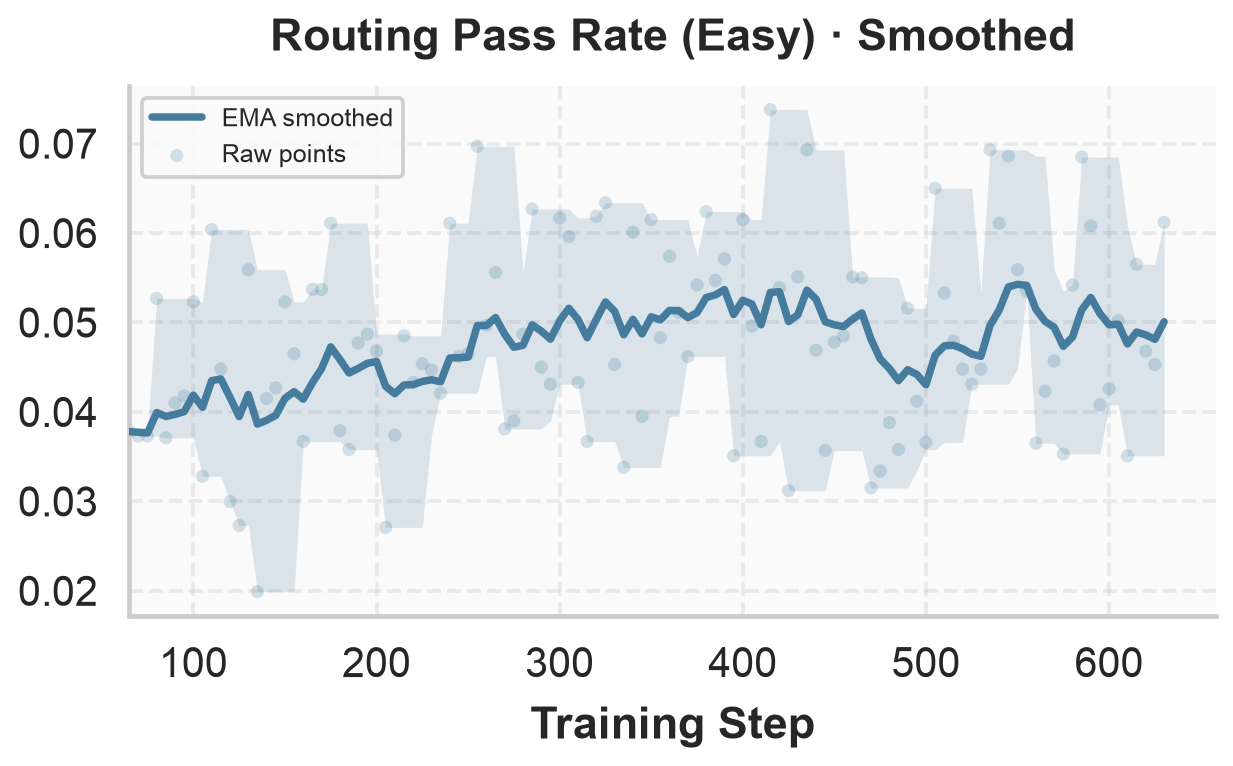}
    \end{subfigure}
    \hfill
    \begin{subfigure}[t]{0.32\textwidth}
        \centering
        \includegraphics[width=\linewidth]{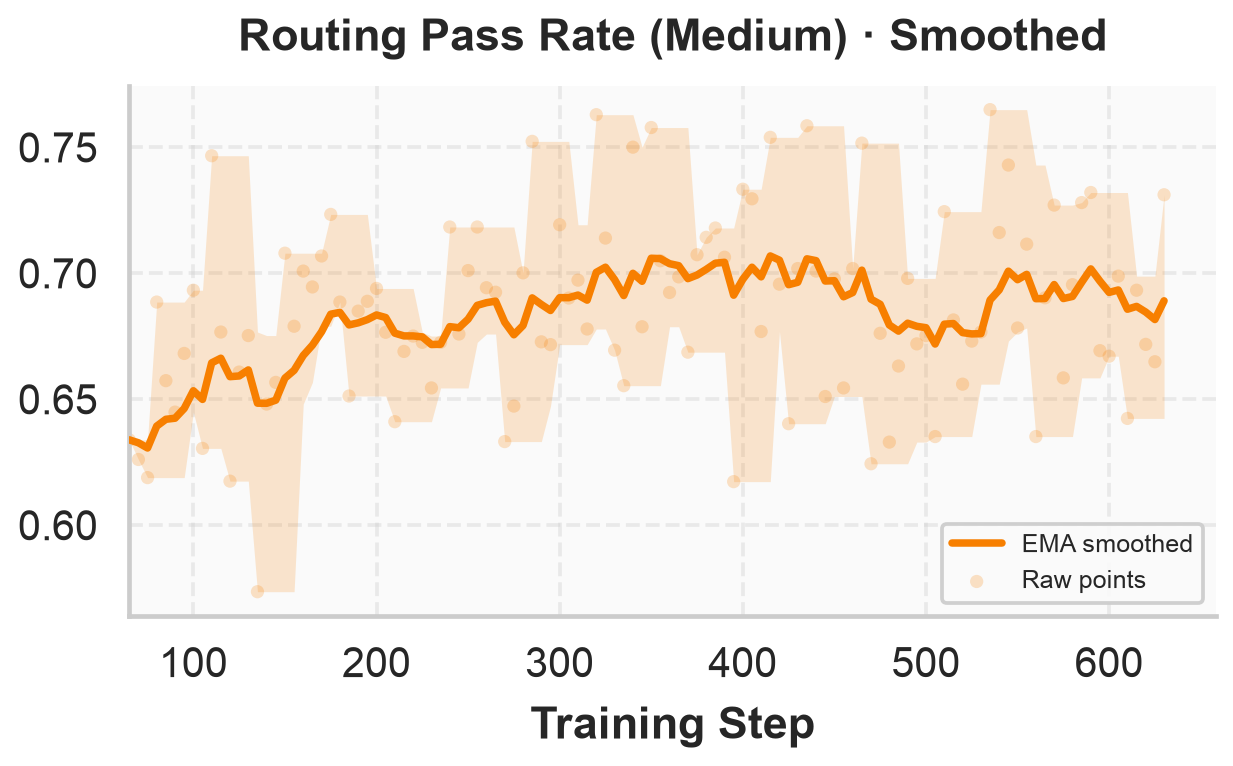}
    \end{subfigure}
    \hfill
    \begin{subfigure}[t]{0.32\textwidth}
        \centering
        \includegraphics[width=\linewidth]{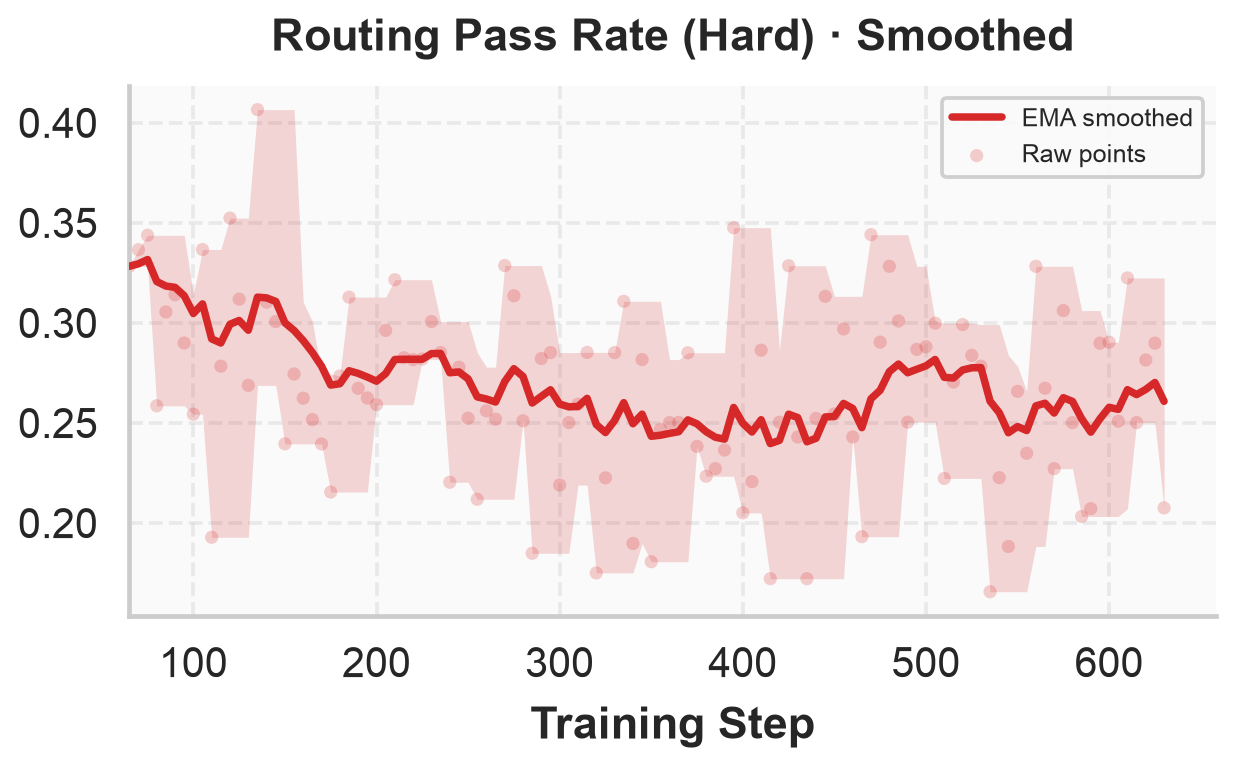}
    \end{subfigure}
    \caption{\textbf{Difficulty-routing dynamics (Qwen3-8B in tooluse)}}
    \label{fig:routing_dynamics}
\end{figure}

\subsection{Ablation Studies}

\paragraph{Component ablation.}
Figure~\ref{fig:tooluse_ablation} compares the mean@16 trajectories of DRIFT and four ablated variants on Tool Use. Because rhythm gating is only defined for the difficulty-routed GRPO branch, we evaluate a joint ablation that removes both difficulty routing and rhythm gating. Among the remaining variants, removing the warm-up stage slows early improvement substantially---the w/o warm-up curve lags behind all others in the first $\sim$50 steps. Removing the success buffer yields unstable mid-training progress between steps 100 and 200, when in-batch positives are scarce and cross-batch replay is most needed. Removing both difficulty routing and rhythm gating leads to a pronounced late-stage decline after step $\sim$200, whereas the complete method continues to improve and finishes with the highest accuracy ($\sim$79\%). These results show that curriculum warm-up, historical success reuse, and the coupled difficulty-routing/rhythm-gating mechanism are all important for stable long-term improvement on tool use.

\begin{figure}[htbp]
    \centering
    \includegraphics[width=0.7\linewidth]{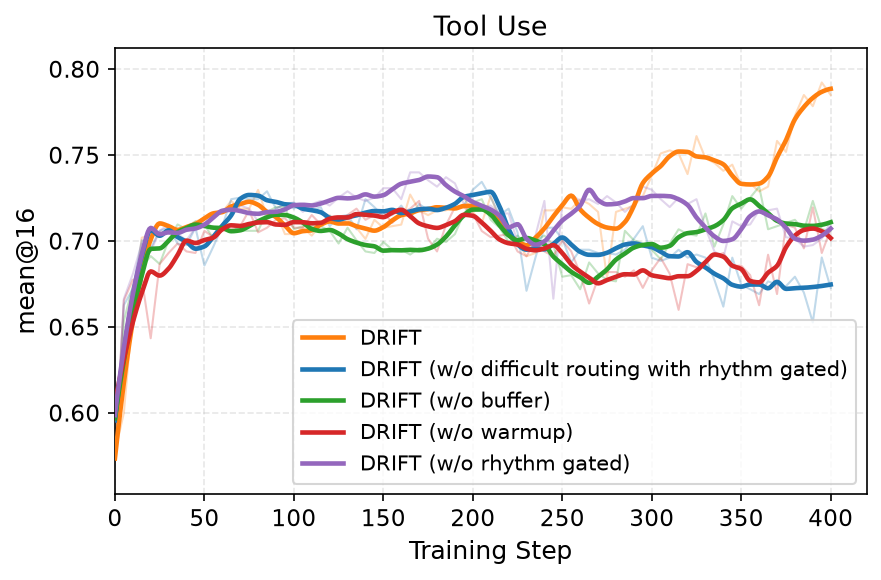}
    \caption{{Ablation study on Tool Use (Qwen3-8B).} Five training curves: full DRIFT and four ablations (w/o warm-up, w/o success buffer, w/o rhythm gating, and w/o difficulty routing with rhythm gating).}
    \label{fig:tooluse_ablation}
\end{figure}

\begin{table}[ht]
\centering
\small
\renewcommand{\arraystretch}{1.25}
\setlength{\tabcolsep}{10pt}
\begin{tabular}{lcccccc}
\toprule
Method & Biology & Chemistry & Material & Physics & Tool Use & Average \\
\midrule
\textbf{Qwen3-8B} & 30.8 & 41.2 & 58.9 & 59.2 & 57.5 & 49.5 \\
+ SDPO & 64.8 & 78.9 & 76.1 & 72.7 & 67.7 & 72.0 \\
+ GRPO & 47.4 & 65.6 & 73.5 & 60.6 & 67.7 & 63.0 \\
\midrule
+ DRIFT w/o difficulty routing & \underline{66.2} & \underline{80.0} & 77.2 & 73.7 & \underline{73.4} & \underline{74.1} \\
+ DRIFT w/o success buffer & 64.9 & 79.9 & \underline{78.8} & \underline{77.8} & 73.2 & 74.9 \\
+ DRIFT w/o warm-up & 58.1 & 78.5 & 75.1 & 74.3 & 72.4 & 71.7 \\
\textbf{+ DRIFT (Ours)} & \textbf{74.4} & \textbf{82.0} & \textbf{81.4} & \textbf{80.5} & \textbf{79.2} & \textbf{79.5} \\
\bottomrule
\end{tabular}
\caption{Ablation study of DRIFT on Qwen3-8B, reported as mean@16 accuracy (\%). \textbf{Bold} values indicate the best result, and \underline{underlined} values indicate the second-best result in each column.}
\label{tab:ablation_qwen3_8b}
\end{table}

\paragraph{Full ablation comparison.}
Table~\ref{tab:ablation_qwen3_8b} confirms that every major component of DRIFT contributes to the final performance. Removing difficulty routing reduces the average score from 79.5\% to 74.1\%, with especially clear drops on biology and tool use, indicating that problem-level allocation of RL and distillation signals is important for focusing optimization on the most learnable region. Removing the success buffer also causes a substantial degradation (74.9\% average), showing that cross-batch reuse of successful trajectories provides important privileged supervision when in-batch positives are sparse. The warm-up stage is likewise necessary: without it, performance falls to 71.7\%, and the drop is particularly severe on biology, suggesting that early-stage capability bootstrapping and buffer accumulation are critical for stable later optimization. Overall, the complete DRIFT design achieves the best result on all five benchmarks, showing that the gains come from the joint effect of curriculum warm-up, historical success reuse, and difficulty-aware routing rather than from any single component alone.

\section{Conclusion}
We addressed the lack of progress-aware signal allocation in LLM self-improvement and proposed DRIFT, an online self-evolution framework that coordinates self-distillation and reinforcement learning through difficulty routing, rhythm-gated exploration, a success buffer, and a two-stage curriculum. On five reasoning benchmarks, DRIFT reaches 79.5\% average accuracy, outperforming the strongest prior baseline by 2.1 points, with the largest gains on tool use.

Effective self-evolution depends less on the choice between reinforcement learning and self-distillation than on when and where each signal is applied. We view problem-level learning-state tracking and token-level credit shaping as general principles for self-evolution, and hope they inform future methods that let language models improve from their own experience without external expert supervision.

\bibliographystyle{plainnat}
\bibliography{references}

\appendix

\section{Further Analysis}

\subsection{Proof of Lemma~\ref{lem:grpo-deadzone}}
\label{sec:app-grpo-deadzone}

Consider a single prompt with $n$ i.i.d.\ rollouts under binary verifiable rewards $r_i\in\{0,1\}$. Let $k=\sum_{i=1}^n r_i$ denote the number of correct rollouts and $\bar r=k/n$ the group mean. Because the rewards are binary, the group standard deviation is
\[
\mathrm{std}(r)
=
\sqrt{\frac{1}{n}\sum_{i=1}^n(r_i-\bar r)^2}
=
\sqrt{\bar r(1-\bar r)}.
\]
GRPO normalizes rewards within the group to form the advantage
\[
A_i=\frac{r_i-\bar r}{\mathrm{std}(r)},
\]
with an all-equal group assigned zero advantage. Hence, $\mathrm{std}(r)=0$ exactly when $k=0$ or $k=n$. In either case, every rollout receives zero advantage, and the group contributes no policy-gradient signal. We refer to such a group as \emph{degenerate}.

If each rollout is correct with probability $p$, then
$k\sim\mathrm{Binomial}(n,p)$, and the probability of obtaining a degenerate group is
\[
P_{\mathrm{deg}}(p)
=
\Pr[k=0]+\Pr[k=n]
=
(1-p)^n+p^n.
\]
Equivalently, the probability that a group contains both correct and incorrect rollouts, and therefore provides a nonzero within-group learning signal, is
\[
P_{\mathrm{inf}}(p)
=
1-P_{\mathrm{deg}}(p)
=
1-(1-p)^n-p^n.
\]
The degeneration probability is symmetric about $p=1/2$, minimized at $p=1/2$, and increases monotonically toward $1$ as $p$ approaches either $0$ or $1$. Consequently, GRPO is most likely to provide an informative learning signal on intermediate-difficulty problems and increasingly likely to lose its within-group signal on problems that are either too hard or too easy. \qed

\paragraph{Implication for routing thresholds.}
The lemma establishes the qualitative need to concentrate full-strength GRPO updates on an intermediate-difficulty region. To calibrate this region quantitatively, we adopt a simple operational requirement: within the medium-difficulty region, at least $80\%$ of sampled rollout groups should contain both correct and incorrect responses and therefore produce a nonzero GRPO advantage.

For the rollout group size $n=8$ used in our experiments, the boundary of this reliability region is determined by
\[
1-(1-p)^8-p^8=0.8.
\]
Solving this equation gives symmetric boundaries of approximately
\[
p_{\mathrm{hard}}\approx0.18,
\qquad
p_{\mathrm{easy}}\approx0.82.
\]
We round these boundaries inward to $0.2$ and $0.8$, respectively, obtaining simple and symmetric routing regions while conservatively preserving the desired signal reliability. At either selected boundary,
\[
P_{\mathrm{inf}}(0.2)
=
P_{\mathrm{inf}}(0.8)
=
1-0.8^8-0.2^8
\approx0.832.
\]
Thus, throughout the interval $p\in[0.2,0.8]$, at least $83.2\%$ of rollout groups are expected to provide a nonzero GRPO signal. In contrast, at $p=0.1$ or $p=0.9$, this probability decreases to
\[
P_{\mathrm{inf}}(0.1)
=
P_{\mathrm{inf}}(0.9)
=
1-0.9^8-0.1^8
\approx0.570.
\]
The thresholds $0.2$ and $0.8$ therefore define a conservative region in which GRPO signals are reliably available, while excluding the hard and easy regimes in which group-level advantages become substantially more likely to degenerate.

\subsection{Proof of Lemma~\ref{lem:rhythm-slope}}
\label{sec:app-rhythm-slope}

\begin{proof}
\emph{Offset invariance.} Adding a constant $d$ to every $H_j$ shifts each of
the two window means by the same $d$, so their difference is unchanged; since
$\mathrm{Drop}_H(t)$ depends on $H$ only through that difference, it is
invariant. In particular the offset $a$ is irrelevant, so we may set $a=0$.

\emph{Slope readout.} On $\mathcal{N}(t)$ we have $H_j=c\,j$. The mean of an
affine sequence over a window equals its value at the window's center index, so
each window mean is $c$ times that center. The two windows are adjacent and each
of width $W$, hence their centers differ by exactly $W$; therefore the
earlier-minus-later gap of window means equals $-cW$. Rectifying,
\[
\mathrm{Drop}_H(t)=\big(\bar H_{\text{earlier}}-\bar H_{\text{later}}\big)_+
=(-cW)_+=W\,(-c)_+,
\]
using $W>0$.
\end{proof}

\subsection{Further Analysis of the Branch Assignment}
\label{sec:app-branch-geometry}

We use a direct-softmax head, so at each position the policy is $p=\mathrm{softmax}(z)$ over the vocabulary $\mathcal{V}$. We state the argument in the unclipped regime $\rho_{i,t}\in[1-\epsilon,1+\epsilon]$ (clipping bounds the step magnitude but not its direction), treat the teacher $q=\pi_\theta(\cdot\mid x,f,y_{i,<t})$ as stop-gradient, and assume binary verifiable rewards.

\paragraph{Negative samples.}
From the softmax identity $\nabla_{z}\log p_a=e_a-p$, the GRPO logit-gradient $A_i(e_a-p)$ has component $A_i(1-p_a)<0$ at the sampled token $a$ (suppression) and component $-A_i\,p_v>0$ at every other token $v\neq a$ (an increase proportional to the current $p_v$). Its direction is thus fixed by $p$ and $a$ alone and is independent of which token is correct: probability is merely drained from $a$ and smeared over the model's existing beliefs. The SDPO term instead descends $\mathrm{JSD}(p\|q)$, which satisfies $\mathrm{JSD}(p\|q)\ge0$ with equality iff $p=q$ and is strictly convex in $p$ for fixed $q$; gradient descent therefore moves $p$ toward $q$, whose conditioning on the privileged feedback $f$ concentrates mass on the correct continuation. Hence only SDPO supplies a corrective target on negatives, and when the whole group is incorrect the GRPO advantage is exactly zero (Lemma~\ref{lem:grpo-deadzone}), leaving SDPO as the sole source of signal.

\paragraph{Positive samples.}
Let $y_i$ and the teacher-inducing sibling $f$ both be correct. Since the reward depends only on final correctness, it is constant on the face $\Delta^{\mathrm{c}}$ of the simplex where all mass lies on correct continuations; hence the reward objective is stationary along every direction tangent to $\Delta^{\mathrm{c}}$, and in particular its gradient does not distinguish among distinct correct solutions. In contrast, for a fixed distinct correct teacher $q\neq p$ the objective $\mathrm{JSD}(p\|q)$ is minimized only at $p=q$, so $-\nabla_\theta\mathrm{JSD}(p\|q)\neq 0$ and points toward that specific correct mode; a first-order expansion around $p=q$ gives $\mathrm{JSD}(p\|q)=\tfrac18\sum_v (p_v-q_v)^2/q_v+o(\|p-q\|^2)$, so the misaligned gradient's norm is monotone in the teacher--student divergence. Thus SDPO on positives injects a reward-irrelevant component that pulls the policy toward one arbitrary correct sibling, reducing solution diversity, whereas the GRPO update $A_i(e_a-p)$ with $A_i>0$ raises $p_a$ for the sample's own correct token---an on-policy, reward-aligned reinforcement.

\subsection{KL, rKL and JSD}
\label{app:divergence}
The choice of divergence plays a fundamental role in on-policy distillation, as it determines how the student policy is encouraged to approach the teacher distribution. Given a teacher policy $\pi_{\mathrm{T}}(\cdot \mid x)$ and a student policy $\pi_{\mathrm{S}}(\cdot \mid x)$ over the vocabulary $\mathcal{V}$, different divergence objectives induce distinct optimization dynamics and consequently lead to different exploration--exploitation behaviors.

The \textbf{forward Kullback--Leibler (KL) divergence}
\[
D_{\mathrm{KL}}
\left(
\pi_{\mathrm{T}}(\cdot \mid x)
\middle\|
\pi_{\mathrm{S}}(\cdot \mid x)
\right)
=
\sum_{a \in \mathcal{V}}
\pi_{\mathrm{T}}(a \mid x)
\log
\frac{
\pi_{\mathrm{T}}(a \mid x)
}{
\pi_{\mathrm{S}}(a \mid x)
}
\]
measures the expected discrepancy under the teacher distribution. Since the optimization emphasizes regions where the teacher assigns high probability, Forward KL strongly penalizes missing probability mass. As a result, the student is encouraged to recover every plausible behavior supported by the teacher, leading to a \textbf{mode-covering} optimization strategy. Such a property is beneficial for preserving behavioral diversity and maintaining exploration. However, because all teacher-supported actions contribute to the objective regardless of their relative quality, Forward KL may continue allocating probability to low-confidence behaviors, slowing policy concentration during later optimization.

In contrast, the \textbf{reverse KL divergence}
\[
D_{\mathrm{KL}}
\left(
\pi_{\mathrm{S}}(\cdot \mid x)
\middle\|
\pi_{\mathrm{T}}(\cdot \mid x)
\right)
=
\sum_{a \in \mathcal{V}}
\pi_{\mathrm{S}}(a \mid x)
\log
\frac{
\pi_{\mathrm{S}}(a \mid x)
}{
\pi_{\mathrm{T}}(a \mid x)
}
\]
evaluates the discrepancy under the student distribution. Actions assigned negligible probability by the student contribute little to the optimization, even if they remain probable under the teacher. Consequently, Reverse KL tends to concentrate probability mass onto a small subset of high-confidence actions, exhibiting the well-known \textbf{mode-seeking} behavior. Although this often accelerates convergence and produces sharper policies, it also increases the risk of premature mode collapse, thereby reducing policy diversity and limiting future exploration.

To balance these two extremes, we consider the \textbf{Jensen--Shannon divergence}. Let the mixture distribution $\pi_{\mathrm{M}}(\cdot \mid x)$ be defined as
\[
\pi_{\mathrm{M}}(\cdot \mid x)
=
\frac{1}{2}
\left(
\pi_{\mathrm{T}}(\cdot \mid x)
+
\pi_{\mathrm{S}}(\cdot \mid x)
\right).
\]
Then the Jensen--Shannon divergence between the teacher and student policies is
\[
D_{\mathrm{JSD}}
\left(
\pi_{\mathrm{T}}(\cdot \mid x),
\pi_{\mathrm{S}}(\cdot \mid x)
\right)
=
\frac{1}{2}
D_{\mathrm{KL}}
\left(
\pi_{\mathrm{T}}(\cdot \mid x)
\middle\|
\pi_{\mathrm{M}}(\cdot \mid x)
\right)
+
\frac{1}{2}
D_{\mathrm{KL}}
\left(
\pi_{\mathrm{S}}(\cdot \mid x)
\middle\|
\pi_{\mathrm{M}}(\cdot \mid x)
\right).
\]

Beyond its distribution-level definition, the Jensen--Shannon divergence also induces a bounded token-level optimization signal through its $f$-divergence form. Define the local probability ratio between the student and teacher policies at token $v \in \mathcal{V}$ as
\[
r(v \mid x)
=
\frac{
\pi_{\mathrm{S}}(v \mid x)
}{
\pi_{\mathrm{T}}(v \mid x)
}.
\]
Equivalently, we parameterize this ratio by a scalar $u(v \mid x)$:
\[
r(v \mid x) = e^{-u(v \mid x)},
\qquad
u(v \mid x)
=
\log \pi_{\mathrm{T}}(v \mid x) - \log \pi_{\mathrm{S}}(v \mid x).
\]
Thus $u>0$ marks tokens that are \emph{under-represented} by the student relative to the teacher, and $u<0$ marks tokens that are \emph{over-represented} by the student.

The Jensen--Shannon divergence can be written as an $f$-divergence with respect to the teacher distribution:
\[
D_{\mathrm{JSD}}
\left(
\pi_{\mathrm{T}},
\pi_{\mathrm{S}}
\right)
=
\mathbb{E}_{v \sim \pi_{\mathrm{T}}}
\left[
f_{\mathrm{JSD}}
\left(
r(v \mid x)
\right)
\right],
\]
where the JSD generator $f_{\mathrm{JSD}}$ is
\[
f_{\mathrm{JSD}}(r)
=
\frac{1}{2}
r
\log
\frac{2r}{1+r}
+
\frac{1}{2}
\log
\frac{2}{1+r}.
\]
Taking the derivative of this generator with respect to $r$ gives
\[
f_{\mathrm{JSD}}'(r)
=
\frac{1}{2}
\log
\frac{2r}{1+r}.
\]

\paragraph{Gradient with respect to the student policy.}
Throughout the derivation the teacher policy is treated as a fixed target (stop-gradient), i.e.\ $\nabla_\theta \pi_{\mathrm{T}} = 0$. Differentiating the $f$-divergence and noting that
\[
\nabla_\theta r(v\mid x)
=
\nabla_\theta \frac{\pi_{\mathrm{S}}(v\mid x)}{\pi_{\mathrm{T}}(v\mid x)}
=
\frac{\pi_{\mathrm{S}}(v\mid x)}{\pi_{\mathrm{T}}(v\mid x)}\,
\nabla_\theta \log \pi_{\mathrm{S}}(v\mid x)
=
r(v\mid x)\,\nabla_\theta \log \pi_{\mathrm{S}}(v\mid x),
\]
the extra factor $r$ converts the teacher expectation into a student expectation:
\[
\nabla_{\theta}
D_{\mathrm{JSD}}
=
\sum_{v}\pi_{\mathrm{T}}(v)\,f_{\mathrm{JSD}}'(r)\,r\,\nabla_\theta\log\pi_{\mathrm{S}}(v)
=
\mathbb{E}_{v \sim \pi_{\mathrm{S}}}
\left[
f_{\mathrm{JSD}}'
\left(
r(v \mid x)
\right)
\nabla_{\theta}
\log
\pi_{\mathrm{S}}(v \mid x)
\right].
\]
Hence the sampled-token gradient coefficient is exactly $f_{\mathrm{JSD}}'(r)$.

\paragraph{Softplus form of the coefficient.}
Substituting $r(v \mid x)=e^{-u(v \mid x)}$ into $f_{\mathrm{JSD}}'(r)$ yields
\[
f_{\mathrm{JSD}}'
\left(
e^{-u}
\right)
=
\frac{1}{2}
\log
\frac{2e^{-u}}{1+e^{-u}}
=
\frac{1}{2}
\log
\frac{2}{1+e^{u}}.
\]
Using the definition
\[
\operatorname{softplus}(u)
=
\log(1+e^u),
\]
we obtain
\[
f_{\mathrm{JSD}}'
\left(
e^{-u}
\right)
=
-\frac{1}{2}
\left(
\operatorname{softplus}(u) - \log 2
\right).
\]
We therefore define
\[
\varphi(u)
=
\frac{1}{2}
\left(
\operatorname{softplus}(u) - \log 2
\right),
\qquad
\text{so that}
\qquad
f_{\mathrm{JSD}}'\!\left(e^{-u}\right) = -\varphi(u).
\]

\paragraph{From gradient coefficient to advantage.}
Since training \emph{minimizes} $D_{\mathrm{JSD}}$, a gradient-descent step moves $\theta$ along $-\nabla_\theta D_{\mathrm{JSD}}$, so the effective weight multiplying $\nabla_\theta \log \pi_{\mathrm{S}}(v)$ in the ascent direction is $-f_{\mathrm{JSD}}'(r) = \varphi(u)$. We therefore adopt this quantity as the token-level advantage,
\[
A_t
=
-f_{\mathrm{JSD}}'\!\left(e^{-u_t}\right)
=
\varphi(u_t),
\]
and write the objective as the stop-gradient surrogate
\[
\mathcal{L}^{\mathrm{JSD}}
=
-\,\mathbb{E}_{v \sim \pi_{\mathrm{S}}}
\left[
\operatorname{sg}(A_t)\,
\log \pi_{\mathrm{S}}(v \mid x)
\right],
\qquad
\nabla_\theta \mathcal{L}^{\mathrm{JSD}} = \nabla_\theta D_{\mathrm{JSD}},
\]
where $u_t$ denotes the value of $u(v \mid x)$ at the sampled token and position $t$. With this convention a \emph{positive} advantage reinforces a token (increases its student log-probability) and a \emph{negative} advantage suppresses it, matching the usual policy-gradient reading.

\paragraph{Directional and boundedness properties.}
The advantage preserves the direction of the discrepancy while controlling the magnitude of the update signal. Since $\operatorname{softplus}(0)=\log 2$ and $\operatorname{softplus}$ is strictly increasing,
\[
\operatorname{sign}(A_t)
=
\operatorname{sign}(\varphi(u_t))
=
\operatorname{sign}(u_t),
\]
so under-represented teacher tokens ($u_t>0$) are reinforced while over-represented student tokens ($u_t<0$) are suppressed. Around $u=0$,
\[
A_t
=
\varphi(u)
\approx
\frac{1}{4}u,
\]
which recovers the local behavior of a KL-based correction up to a constant scale. The two tails, however, are markedly asymmetric. As $u\to-\infty$,
\[
\varphi(u)\;\longrightarrow\;-\tfrac{1}{2}\log 2,
\]
so $\varphi(u) \ge -\tfrac{1}{2}\log 2$ and the \emph{suppression} applied to over-represented tokens is bounded in magnitude by $\tfrac{1}{2}\log 2$. As $u\to+\infty$, in contrast,
\[
\varphi(u)\approx\tfrac{1}{2}u\;\longrightarrow\;+\infty,
\]
so \emph{reinforcement} of under-represented teacher tokens grows without bound (asymptotically linearly).

This asymmetry is exactly what a symmetric matching objective should exhibit. The bounded negative branch prevents tokens in the heavily over-represented tail of $\pi_{\mathrm{S}}$ from receiving unbounded penalties; unlike Reverse KL, whose per-token weight diverges to $-\infty$ on such tokens and produces large destabilizing gradients, JSD caps how aggressively the student erases its own probability mass, preserving diversity and improving numerical stability. Meanwhile, the unbounded, mode-covering positive branch keeps pulling up teacher-preferred behaviors that the student currently misses, which is the term that actually guards against premature mode collapse.

At the distribution level, JSD provides a complementary form of balance. Rather than directly measuring $\pi_{\mathrm{T}}$ against $\pi_{\mathrm{S}}$, it compares both policies with respect to their shared mixture distribution $\pi_{\mathrm{M}}$. As a result, deviations from either side are treated symmetrically, avoiding the one-sided bias of purely Forward-KL or Reverse-KL objectives.

This property is especially important in on-policy self-distillation, where the teacher $\pi_{\mathrm{T}}$ and the student $\pi_{\mathrm{S}}$ evolve simultaneously throughout training and neither policy should be regarded as an immutable target distribution. Concretely, at each step we treat the teacher as a stop-gradient copy of the current (or slowly-updated) policy, so that the gradient derivation above holds step-wise even though the target drifts across steps. A symmetric divergence therefore avoids anchoring the optimization too strongly to either policy, while still maintaining sufficient exploration and preventing overly aggressive policy collapse. For these reasons, we adopt the Jensen--Shannon divergence as the distribution matching objective in our framework.

\subsection{PG-style OPD vs. GKD-style OPD}
Although the sampled-token log-ratio is one of the most common implementation forms in current on-policy distillation (OPD), it does not characterize the entire OPD family. More fundamentally, existing OPD methods can be divided according to how gradients are propagated from tokens back to model parameters, rather than simply by whether they use reverse KL or another divergence objective.

\paragraph{PG-style OPD: Policy Gradient / REINFORCE.}
PG-style OPD reformulates distillation as a policy-gradient reinforcement learning problem. The model first samples trajectories from the current student policy,
\[
y_t \sim p_t(\cdot),
\]
and then constructs a token-level reward from the log-ratio between the teacher and the student on the sampled token:
\[
r_t
=
\operatorname{sg}
\left[
\log q_t(y_t) - \log p_t(y_t)
\right],
\]
where both the teacher output and the reward are treated with stop-gradient. The optimization objective can be written as the REINFORCE surrogate
\[
\mathcal{L}^{\mathrm{PG}}
=
-
\mathbb{E}_{y \sim p}
\sum_t
r_t \log p_t(y_t).
\]
Under this framework, gradients are propagated only through $\log p_t(y_t)$ for the sampled token. Therefore, PG-style OPD corresponds to a sparse, single-sample gradient estimator, which typically has high variance and often requires techniques such as advantage normalization, baselines, or reward shaping for stabilization. Representative methods include REOPOLD, which interprets the log-ratio as a token reward, AOPD, which optimizes positive and negative advantage branches, and the policy-gradient form of OPSD.

\paragraph{GKD-style OPD: Loss-form / Generalized Knowledge Distillation.}
Another class of OPD methods preserves the supervised-learning form of distillation. On trajectories sampled from the student policy, the teacher only provides a stop-gradient supervision signal, while the training objective directly minimizes a distributional divergence at each position:
\[
\mathcal{L}^{\mathrm{GKD}}
=
\mathbb{E}_{y \sim p}
\sum_t
D
\left(
p_t(\cdot)
\,\middle\|\,
q_t(\cdot)
\right),
\qquad
q_t
=
\operatorname{sg}(\text{teacher}).
\]
Unlike PG-style OPD, the gradient does not come only from the sampled token, but is propagated through the predicted probabilities $p_t(a)$ over the entire vocabulary, or over a Top-$K$ truncated candidate set. Thus, this class of methods yields dense, analytic-expectation gradients with substantially lower variance, making it closer to traditional knowledge distillation. Depending on the divergence definition, $D(\cdot)$ can be instantiated as Forward KL, Reverse KL, Jensen--Shannon divergence (JSD), or their generalized variants. Representative methods include classical generalized knowledge distillation (GKD), the Top-$K$ Reverse-KL objective proposed in \emph{The Many Faces of Distillation}, and most OPD implementations based on distillation losses.

\paragraph{Comparison.}
The core difference between the two classes does not lie in which divergence is adopted, but in the gradient propagation mechanism. PG-style OPD treats on-policy distillation as policy optimization driven by log-ratio rewards. Its gradient depends only on sampled tokens and therefore corresponds to sparse, single-sample estimation. This leads to higher variance, but allows the method to be naturally integrated into reinforcement learning frameworks.

In contrast, GKD-style OPD formulates on-policy distillation as differentiable distribution matching on on-policy trajectories. Its gradient directly acts on the full output distribution, yielding dense, low-variance analytic gradients and more stable training.

Therefore, from an optimization perspective, the essential distinction is not reward optimization versus distribution matching, but rather two different gradient computation paradigms: sample-based gradient estimation through policy gradients and distribution-level differentiation through analytic gradients. This distinction is the fundamental reason why existing OPD methods differ in optimization stability, gradient variance, and computational efficiency.

\section{Related Works}
\subsection{RLVR}
Aligning and adapting large language models (LLMs) increasingly relies on post-training with verifiable rewards, a strategy grounded in classic policy-gradient algorithms such as REINFORCE and PPO~\cite{williams1992simple,schulman2017proximal}. An expanding body of literature builds upon these principles for LLM optimization, where trajectories sampled from the model are optimized using sequence-level outcome rewards~\cite{guo2025deepseekr1,shao2024deepseekmath,yu2026dapo,liu2025r1zero,zheng2025gspo,zhoubian2025restrl}. Among these approaches, GRPO has emerged as a robust and scalable baseline by estimating advantages from group-relative rewards without requiring a separate critic network~\cite{shao2024deepseekmath}.

A major limitation of these methods, however, is their coarse-grained credit assignment, which distributes a single scalar advantage uniformly across all generated tokens. Recent studies have shown that such homogeneous supervision leads to gradient dilution among causally unrelated tokens~\cite{khandoga2026beyond}, introduces sequence-length-dependent optimization biases~\cite{parthasarathi2025grpo}, and fails to accurately localize semantic errors in otherwise correct code generations~\cite{kumar2026execution}. To address this issue, several recent frameworks leverage process supervision and process reward models (PRMs) to provide dense, step-level supervisory signals derived from intermediate reasoning trajectories~\cite{lightman2024lets,setlur2025rewarding,cui2025process}. Although these methods substantially improve credit assignment resolution, they rely on additional trained reward evaluators, increasing both annotation and computational costs. This limitation motivates approaches that can provide dense supervision without requiring auxiliary reward models.

\subsection{Dense Supervision in On-Policy Distillation}

Recent advances in on-policy distillation have demonstrated that dense token-level supervision provides substantially richer optimization signals than sparse outcome-level reinforcement learning, leading to improved sample efficiency and reasoning performance \cite{agarwal2024opd,gu2024gkd,lu2025opd}. However, growing evidence suggests that the effectiveness of this paradigm is fundamentally determined by the quality of the teacher-provided supervision rather than the distillation objective itself. Recent studies therefore increasingly focus on improving the reliability, informativeness, and optimization behavior of token-level guidance from different perspectives.

One line of work revisits the optimization formulation of on-policy distillation. G-OPD interprets OPD as a special case of KL-regularized reinforcement learning and generalizes it through reward extrapolation and flexible reference policies, enabling students to surpass the capability boundary of their teachers \cite{yang2026gopd}. REOPOLD further analyzes OPD from a policy optimization perspective and attributes training instability to overly aggressive imitation, proposing relaxed optimization strategies including reward clipping, entropy-guided token sampling, and exploration-to-refinement curricula to improve stability and sample efficiency \cite{ko2026reopold}.

Another line of work focuses on improving the fidelity of teacher supervision. EOPD observes that conventional reverse-KL distillation collapses generation diversity and provides unstable gradients on high-entropy tokens, motivating an adaptive combination of reverse and forward KL to preserve uncertainty while maintaining efficient learning \cite{jin2026entropyopd}. Meanwhile, recent analyses of self-distillation reveal that privileged-context teachers may gradually suppress epistemic verbalization and uncertainty-aware reasoning, causing performance degradation on challenging reasoning tasks despite shorter reasoning traces \cite{kim2026selfdistill}. In agentic settings, SDAR further shows that multi-turn interaction amplifies teacher--student divergence and makes privileged supervision inherently asymmetric, advocating confidence-aware gating that selectively trusts teacher guidance according to token-level reliability \cite{lu2026sdar}.

Overall, existing studies consistently indicate that improving dense supervision is not merely a matter of designing better optimization objectives, but also requires accurately characterizing when teacher signals are reliable, uncertain, or even misleading. Our work follows this direction from a complementary perspective. Rather than modifying the distillation objective itself or relying on uncertainty estimated from the teacher, we explicitly model the learning dynamics of each training sample and adaptively route optimization signals according to sample difficulty, enabling dense supervision to remain informative throughout the entire post-training process.

\section{Experimental Details}
\subsection{Technical setup}
All experiments were conducted on a single node equipped with 8 NVIDIA H200 GPUs interconnected via NVLink, providing a total of 1123 GB VRAM. Our software environment
uses GPU driver version 570.158.01, CUDA 12.9, and PyTorch 2.10.0+cu129.

Our implementation is built on the verl reinforcement learning framework \cite{sheng2025hybridflow}. We use PyTorch Fully Sharded Data Parallel (FSDP) for distributed policy optimization. For rollout generation, we employ vLLM \cite{kwon2023efficient} through verl's hybrid actor--rollout engine, which enables efficient batched inference across the multi-GPU node.

To ensure a fair comparison, GRPO and SDPO baselines largely follow the hyperparameter settings of SDPO~\citep{hubotter2026reinforcement}; DRIFT uses the same settings except where noted below. Table~\ref{tab:hyperparameters} summarizes the full configuration.

\subsection{Hyperparameters}

Table~\ref{tab:hyperparameters} lists the training hyperparameters for GRPO, SDPO, and DRIFT. Parameters that do not apply to a method are marked with ``---''.

\begin{table*}[htbp]
\centering
\footnotesize
\setlength{\tabcolsep}{4pt}
\renewcommand{\arraystretch}{1.15}
\caption{Training hyperparameters for GRPO, SDPO, and DRIFT.}
\label{tab:hyperparameters}
\begin{tabular*}{0.7\textwidth}{@{\extracolsep{\fill}}lccc@{}}
\toprule
\textbf{Parameters} & \textbf{GRPO} & \textbf{SDPO} & \textbf{DRIFT} \\
\midrule
\multicolumn{4}{l}{\textbf{General}} \\
Model & Qwen3-\{4B,8B\} & Qwen3-\{4B,8B\} & Qwen3-\{4B,8B\} \\
 & OLMo3-7B & OLMo3-7B & OLMo3-7B \\
Thinking & False & False & False \\
\midrule
\multicolumn{4}{l}{\textbf{Data}} \\
Max.\ prompt length & 2048 & 2048 & 2048 \\
Max.\ response length & 8192 & 8192 & 8192 \\
\midrule
\multicolumn{4}{l}{\textbf{Batching}} \\
Question batch size & 32 & 32 & 32 \\
Mini batch size & 32 & 32 & 32 \\
Number of rollouts & 8 & 8 & 8 \\
\midrule
\multicolumn{4}{l}{\textbf{Rollout}} \\
Inference engine & vLLM & vLLM & vLLM \\
Temperature & 1.0 & 1.0 & 1.0 \\
\midrule
\multicolumn{4}{l}{\textbf{Validation}} \\
Number of rollouts & 16 & 16 & 16 \\
Temperature & 0.6 & 0.6 & 0.6 \\
Top-$p$ & 0.95 & 0.95 & 0.95 \\
\midrule
\multicolumn{4}{l}{\textbf{SDPO loss}} \\
Top-$K$ distillation & --- & 100 & 100 \\
Distillation divergence & --- & Jensen--Shannon & Jensen--Shannon \\
Teacher-EMA update rate & --- & 0.05 & 0.05 \\
Rollout IS clip ($\rho$) & --- & 2 & 2 \\
\midrule
\multicolumn{4}{l}{\textbf{Two-stage training}} \\
Warm-up training steps & --- & --- & 64 \\
\midrule
\multicolumn{4}{l}{\textbf{Difficulty routing}} \\
$\gamma_{\mathrm{hard}}$ & --- & --- & 0 \\
$\gamma_{\mathrm{easy}}$ & --- & --- & 0.5 \\
$p_{\mathrm{easy}}$ & --- & --- & 0.8 \\
$p_{\mathrm{hard}}$ & --- & --- & 0.2 \\
\midrule
\multicolumn{4}{l}{\textbf{Rhythm gate}} \\
Local window length $W$ & --- & --- & 10 \\
\midrule
\multicolumn{4}{l}{\textbf{Success buffer}} \\
Buffer size per uid & --- & --- & 3 \\
Buffer accumulation epochs & --- & --- & 3 \\
\midrule
\multicolumn{4}{l}{\textbf{Training}} \\
Optimizer & AdamW & AdamW & AdamW \\
Learning rate & $1{\times}10^{-5}$ & $1{\times}10^{-5}$ & $5{\times}10^{-6}$ \\
Warmup steps & 10 & 10 & 10 \\
Weight decay & 0.01 & 0.01 & 0.01 \\
Gradient clip norm & 1.0 & 1.0 & 1.0 \\
\bottomrule
\end{tabular*}
\end{table*}

\subsection{Prompt Templates}\label{sec:app-prompts}

We use the same prompt templates as SDPO~\citep{hubotter2026reinforcement} without any modification, ensuring a fair comparison across all methods. The Science Q\&A benchmarks (Chemistry, Physics, Biology, Materials) share a common multiple-choice format, while Tool Use follows a separate tool-calling format. For Tool Use, the full tool specification in each user prompt is identical to SDPO; Table~\ref{tab:prompt-templates} shows the response-format skeleton only.

\begin{table}[htbp]
\centering
\footnotesize
\renewcommand{\arraystretch}{1.2}
\setlength{\tabcolsep}{4pt}
\caption{Prompt templates used in all experiments (identical to SDPO).}
\label{tab:prompt-templates}
\begin{tabular}{p{1.8cm}p{1.2cm}p{11.5cm}}
    \toprule
    \textbf{Benchmark} & \textbf{Role} & \textbf{Prompt} \\
    \midrule
    Science Q\&A & System &
    \begin{minipage}[t]{11.3cm}\ttfamily\footnotesize
Given a question and four options, please select the right answer. Respond in the following format:\\
<reasoning>\\
\ldots\\
</reasoning>\\
<answer>\\
\ldots\\
</answer>\\[4pt]
For the answer, only output the letter corresponding to the correct option (A, B, C, or D), and nothing else. Do not restate the answer text. For example, if the answer is ``A'', just output:\\
<answer>\\
A\\
</answer>
    \end{minipage} \\
    Science Q\&A & User &
    \begin{minipage}[t]{11.3cm}\ttfamily\footnotesize
\{question\}\\
Please reason step by step.
    \end{minipage} \\
    Tool Use & System &
    \begin{minipage}[t]{11.3cm}\ttfamily\footnotesize
You are a tool-use assistant. Solve each request by reasoning about the task and calling the provided tools when needed.\\
Use only the tools provided in the user message.\\
Follow the required response format exactly.
    \end{minipage} \\
    Tool Use & User &
    \begin{minipage}[t]{11.3cm}\ttfamily\footnotesize
Your task is to answer the user's question using available tools.\\
You have access to the following tools:\\
\lbrack\{tool name, description, parameters, and outputs\}\rbrack\\[4pt]
Use the following format:\\
Thought: \ldots\\
Action: <tool name>\\
Action Input: <JSON>\\[4pt]
Begin!\\
Question: \{user query\}
    \end{minipage} \\
    \bottomrule
\end{tabular}
\end{table}

\subsection{Benchmark Details}\label{sec:app-benchmarks}

We use the exact train/test splits provided in the official SDPO GitHub repository to ensure full comparability. Table~\ref{tab:benchmark-stats} summarizes the dataset statistics.

\begin{table}[htbp]
\centering
\small
\renewcommand{\arraystretch}{1.15}
\setlength{\tabcolsep}{9pt}
\caption{Dataset statistics for all five benchmarks. The four Science Q\&A benchmarks are drawn from the reasoning subset (Level 3) of SciKnowEval~\citep{feng2024sciknoweval}; Tool Use is drawn from ToolAlpaca~\citep{tang2023toolalpaca}. All splits are identical to those used by SDPO~\citep{hubotter2026reinforcement}.}
\label{tab:benchmark-stats}
\begin{tabular}{llrrr}
    \toprule
    \textbf{Benchmark} & \textbf{Source} & \textbf{Train} & \textbf{Test} & \textbf{Total} \\
    \midrule
    Chemistry & SciKnowEval & 1{,}890 & 210 & 2{,}100 \\
    Physics & SciKnowEval & 720 & 80 & 800 \\
    Biology & SciKnowEval & 450 & 50 & 500 \\
    Materials & SciKnowEval & 841 & 94 & 935 \\
    Tool Use & ToolAlpaca & 4{,}046 & 68 & 4{,}114 \\
    \bottomrule
\end{tabular}
\end{table}

The four Science Q\&A benchmarks are formatted as four-option single-choice questions targeting undergraduate-level scientific reasoning. Each question presents a problem statement (often involving domain-specific notation such as SMILES strings in Chemistry, physical equations in Physics, protein sequences in Biology, or crystal lattice parameters in Materials) followed by four candidate answers. The Tool Use benchmark pairs a natural-language user request with a tool-API specification (including function names, parameter schemas, and output types); the model must produce the correct tool call in a structured Thought/Action/Action Input format.

Table~\ref{tab:benchmark-examples} shows one representative example from each benchmark.
\begin{table}[htbp]
\centering
\footnotesize
\renewcommand{\arraystretch}{1.25}
\setlength{\tabcolsep}{4pt}
\caption{One representative example from each benchmark. Science Q\&A examples show the question stem and four answer options; the Tool Use example shows the user query and the expected structured tool call (API specification omitted for brevity; see Table~\ref{tab:prompt-templates} for the prompt template).}
\label{tab:benchmark-examples}
\begin{tabular}{@{}p{1.5cm}p{10.2cm}p{3.3cm}@{}}
    \toprule
    \textbf{Benchmark} & \textbf{Question (excerpt)} & \textbf{Answer} \\
    \midrule
    Chemistry &
    \begin{minipage}[t]{10cm}
    What is the correct logarithmic solubility value of the molecule ``Cc1cc(=O)[nH]c(=S)[nH]1'' in aqueous solutions?\\[2pt]
    A: $-3.01$ \quad B: $-2.436$\\
    C: $-4.576$ \quad D: $1.1$
    \end{minipage} & B \\
    Physics &
    \begin{minipage}[t]{10cm}
    A charged particle produces an electric field with a magnitude of $2.0\;\mathrm{N/C}$ at a point that is $50\;\mathrm{cm}$ away from the particle. What is the magnitude of the particle's charge?\\[2pt]
    A: 50 pC \quad B: 56 pC\\
    C: 60 pC \quad D: 64 pC
    \end{minipage} & B \\
    Biology &
    \begin{minipage}[t]{10cm}
    What is the folding stability score of the protein sequence ``GSSTTRYRFLDEEEARRAAKEWARRGYQVHVTQNGTYWEVEVR''?\\[2pt]
    A: $-0.01$ \quad B: $1.69$\\
    C: $2.49$ \quad D: $0.45$
    \end{minipage} & B \\
    Materials &
    \begin{minipage}[t]{10cm}
    Given the following crystal structure parameters for the material RbLa\textsubscript{9}(IrO\textsubscript{6})\textsubscript{4} (Material ID: mp-560657), calculate the volume of the unit cell (in \AA\textsuperscript{3}). Lattice: $a{=}7.82$, $b{=}7.82$, $c{=}17.88$\,\AA; $\alpha{=}\beta{=}\gamma{=}90^\circ$.\\[2pt]
    A: 1025.67 \quad B: 1094.31\\
    C: 1200.45 \quad D: 1150.78
    \end{minipage} & B \\
    Tool Use &
    \begin{minipage}[t]{10cm}
    \textit{(Given the Axolotl API specification)}\\[2pt]
    Question: ``I'm looking for an axolotl that is wild in color and medium in size. Can you help me find some pictures?''
    \end{minipage} &
    \begin{minipage}[t]{3.1cm}\ttfamily\scriptsize
    Action: searchAxolotlImages\\
    Action Input: \{"color": "wild", "gender": "", "size": "medium", "page": 1\}
    \end{minipage} \\
    \bottomrule
\end{tabular}
\end{table}
\section{Additional Experiment Results}
\subsection{Additional Rhythm-Gating Visualizations}
\label{sec:app-rhythm-vis}

Figure~\ref{fig:rhythm-vis-tooluse} provides a Tool-Use counterpart to the Physics example in Figure~\ref{fig:rhythm-vis-physics}.
The same red/blue scheme is used: red for tokens kept by the gate ($b_t^{\mathrm{reb}}g_t^{\mathrm{rhythm}}$) and blue for rebellious exceedances suppressed by the gate ($b_t^{\mathrm{reb}}(1-g_t^{\mathrm{rhythm}})$).
On this trajectory, gated credit concentrates on schema-critical tokens (e.g., action names and argument values), while many formatting and connective exceedances are filtered out.

\begin{figure}[ht]
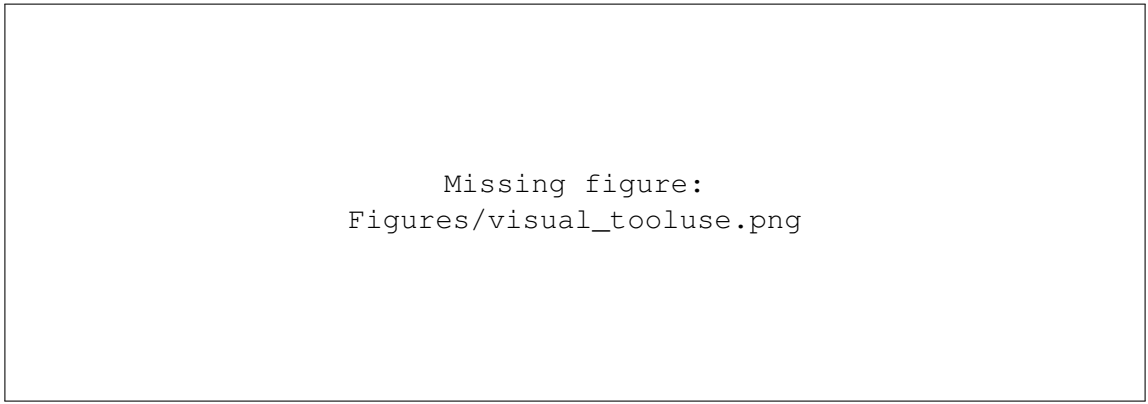

    \centering
    \includefig[width=0.9\linewidth]{Figures/visual_tooluse.png}
    \caption{Rhythm-gating visualization on a correct Tool-Use rollout
    (Qwen3-8B trained with DRIFT on Tool Use).
    Color coding matches Figure~\ref{fig:rhythm-vis-physics}.}
    \label{fig:rhythm-vis-tooluse}
\end{figure}

\subsection{Validation Performance}
Figure~\ref{fig:appendix_sciknoweval_validation} shows the validation curves of DRIFT on the STEM datasets. Both mean@16 and best@16 improve steadily throughout training, indicating that the gains are not confined to a single checkpoint.
\begin{figure*}[ht]
    \centering
    \includegraphics[width=0.90\textwidth]{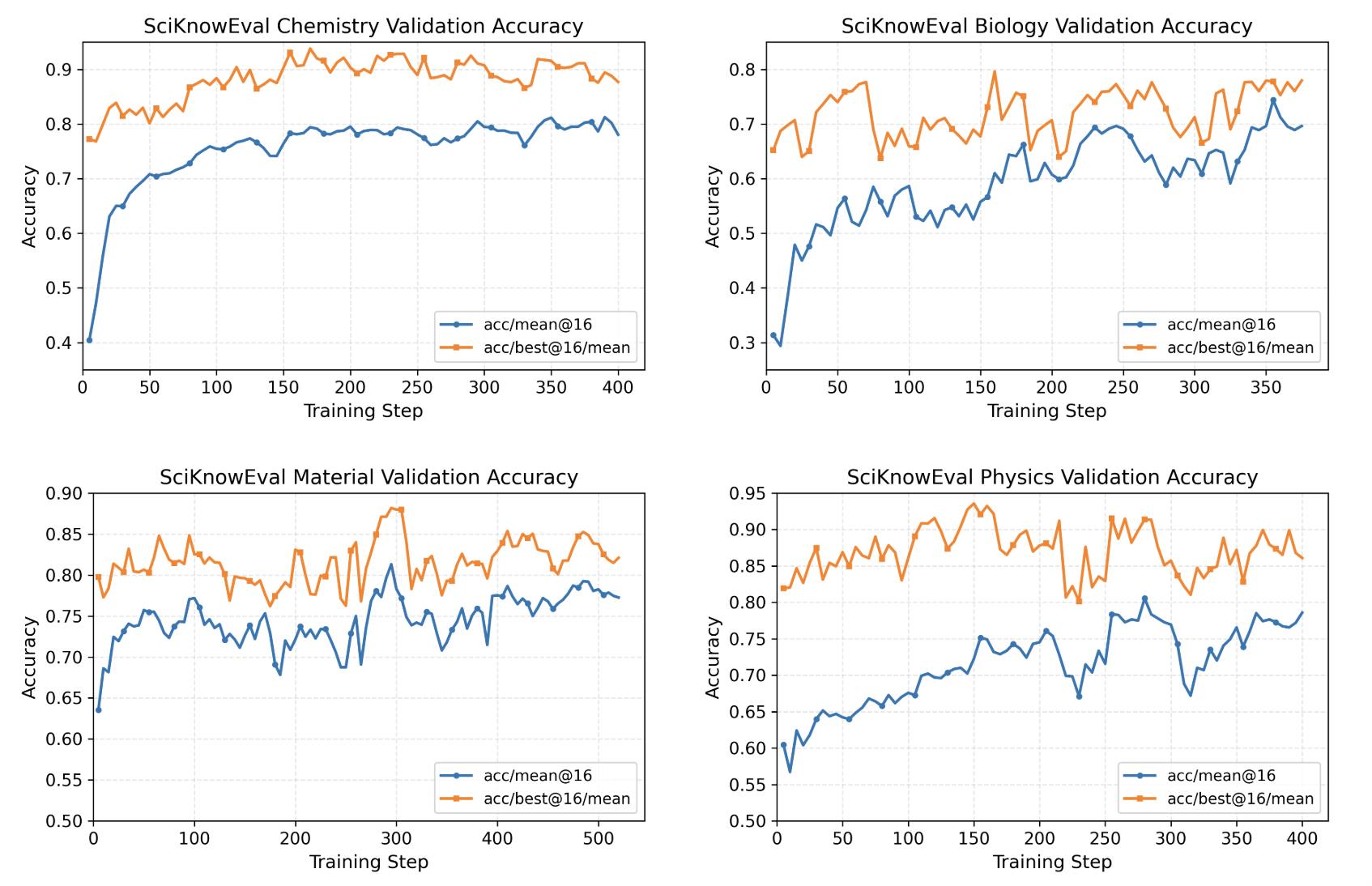}
    \caption{\textbf{Validation performance on the STEM datasets.} In addition to improvements in mean@16, best@16 also increases steadily throughout training.}
    \label{fig:appendix_sciknoweval_validation}
\end{figure*}

\subsection{Best@16 Comparison}
Table~\ref{tab:appendix_best16_comparison} reports best@16 accuracy on Qwen3-8B. DRIFT outperforms SDPO and GRPO on all five tasks, with particularly pronounced gains on materials and tool use.
\begin{table}[ht]
    \centering
    \small
    \caption{\textbf{Best@16 performance on Qwen3-8B.} DRIFT outperforms SDPO and GRPO across all five tasks, with particularly pronounced gains on materials and tool use.}
    \label{tab:appendix_best16_comparison}
    \begin{tabular}{lccccc}
        \toprule
        Method & Biology & Chemistry & Material & Physics & Tool Use \\
        \midrule
        SDPO  & 78.8 & 89.4 & 80.4 & 91.6 & 73.4 \\
        GRPO  & 80.0 & 93.5 & 82.9 & 87.2 & 70.8 \\
        \textbf{DRIFT} & \textbf{80.1} & \textbf{94.3} & \textbf{87.9} & \textbf{93.8} & \textbf{83.7} \\
        \bottomrule
    \end{tabular}
\end{table}

\subsection{More Training Dynamics}
Figure~\ref{fig:appendix_more_training_dynamics} extends the Tool-Use dynamics of Section~\ref{sec:training_dynamics} to biology, chemistry, and materials. The same qualitative trends---bounded response length, growing success-buffer reuse, and a shift from hard toward easy/medium routing---appear across these STEM tasks.
\begin{figure*}[t]
    \centering
    \scriptsize
    \setlength{\tabcolsep}{10pt}
    \renewcommand{\arraystretch}{0.6}
    \begin{tabular}{ccc}
        \textbf{Biology} & \textbf{Chemistry} & \textbf{Material} \\
        \includegraphics[width=0.30\textwidth,height=0.105\textheight,keepaspectratio]{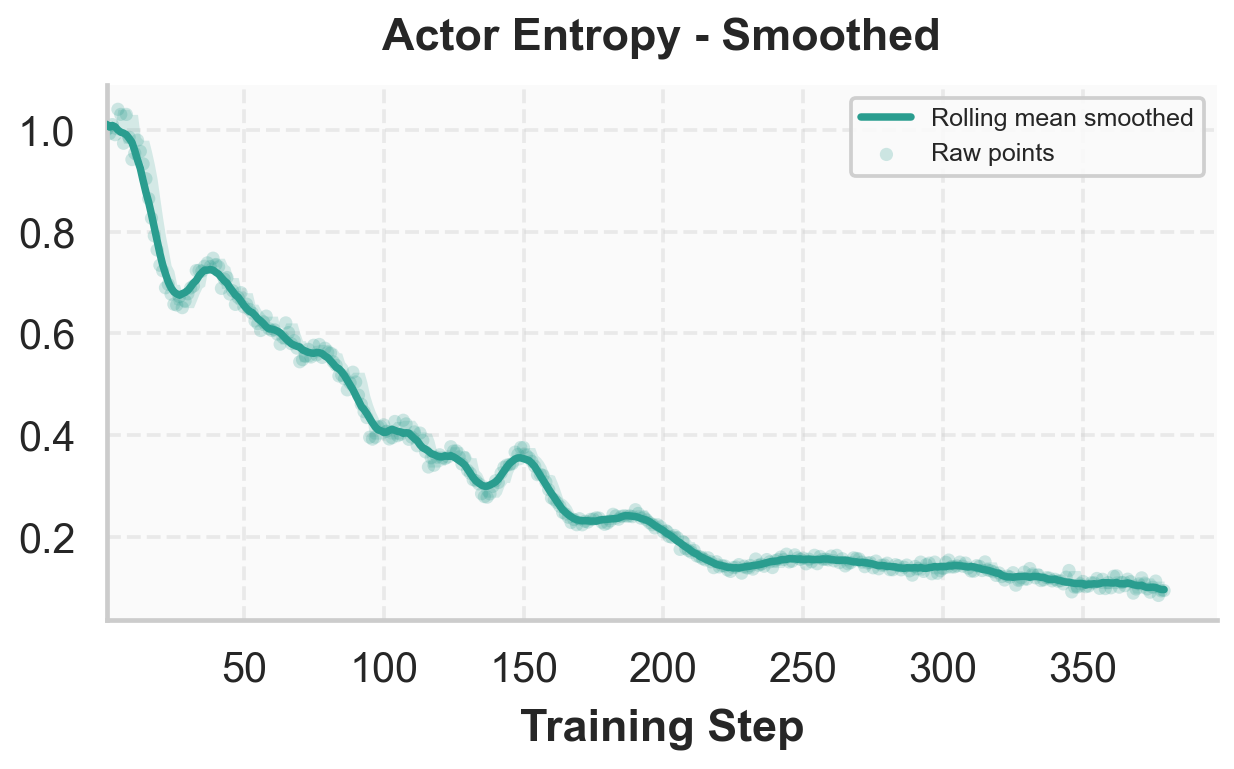} &
        \includegraphics[width=0.30\textwidth,height=0.105\textheight,keepaspectratio]{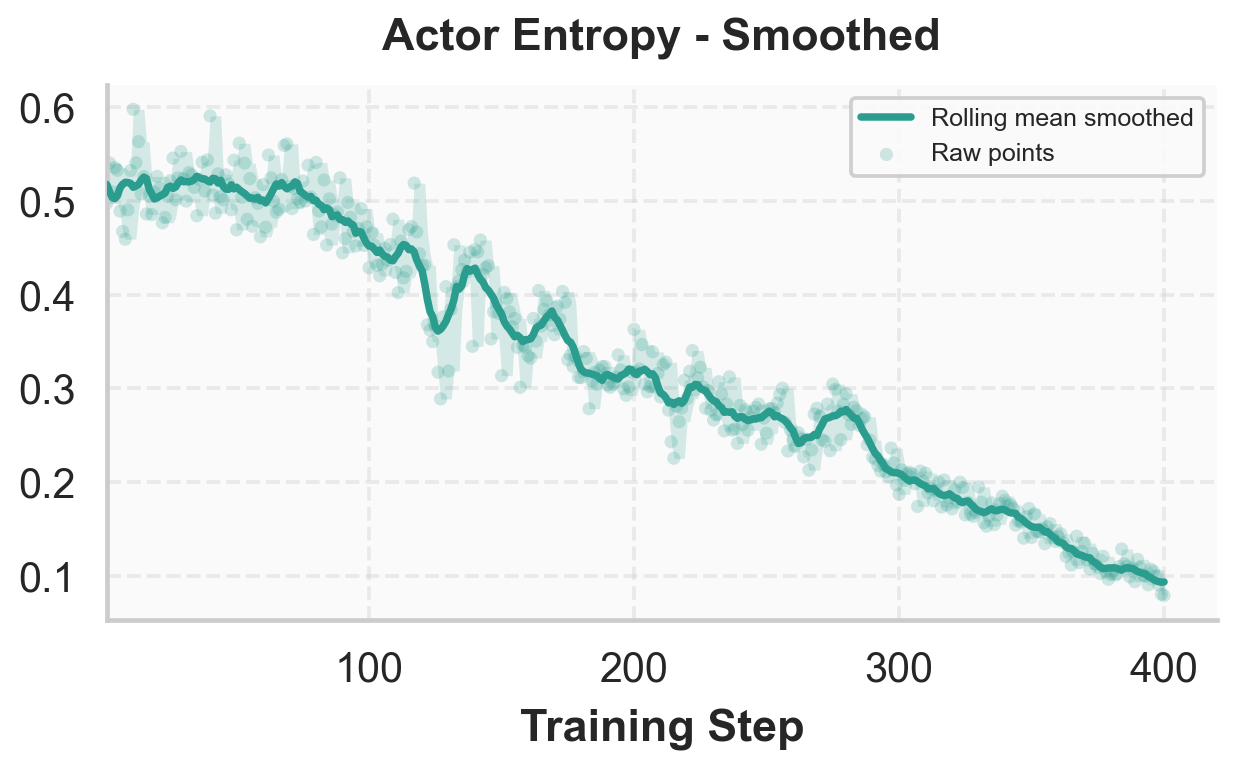} &
        \includegraphics[width=0.30\textwidth,height=0.105\textheight,keepaspectratio]{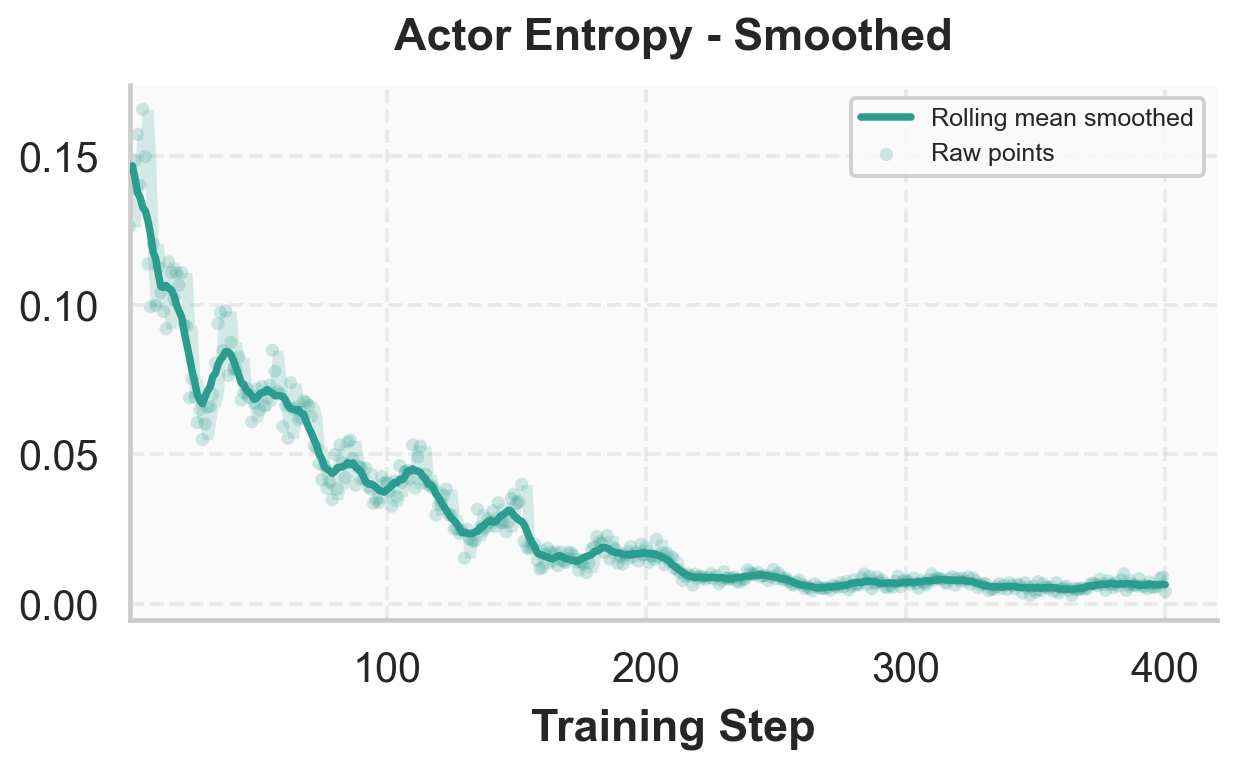} \\
        \includegraphics[width=0.30\textwidth,height=0.105\textheight,keepaspectratio]{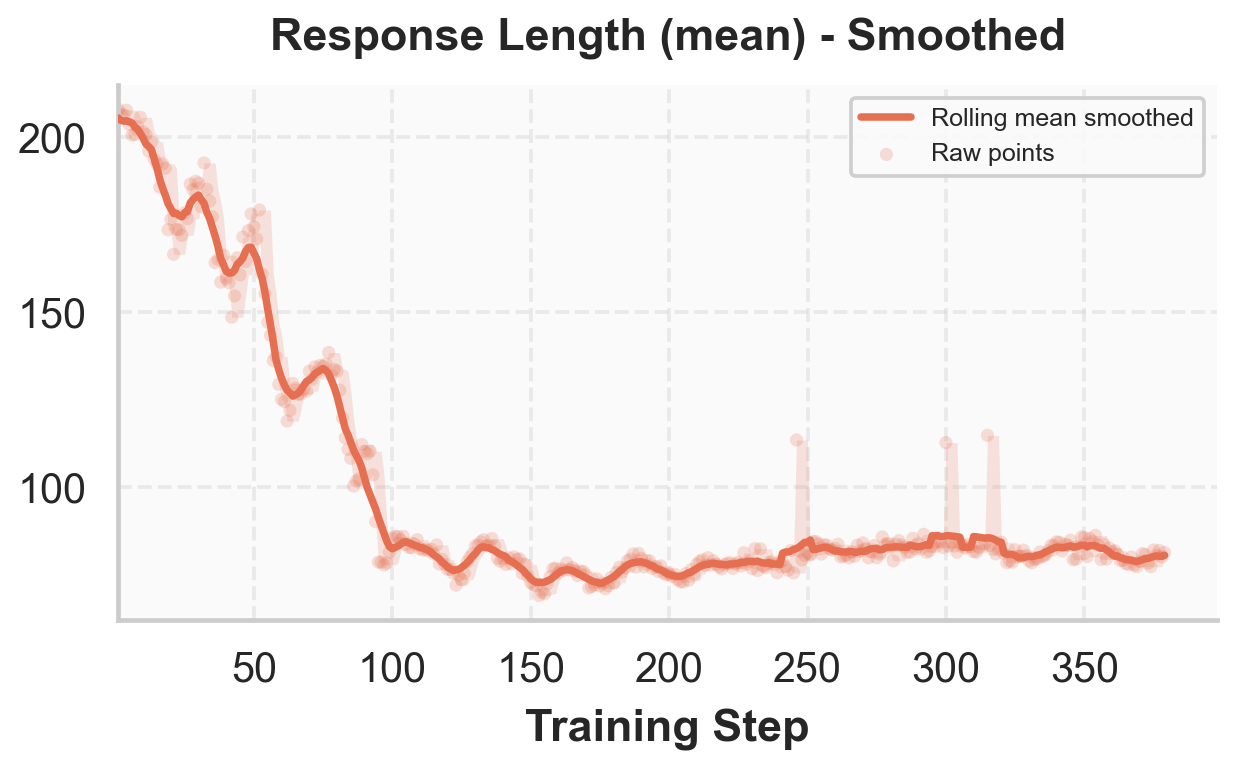} &
        \includegraphics[width=0.30\textwidth,height=0.105\textheight,keepaspectratio]{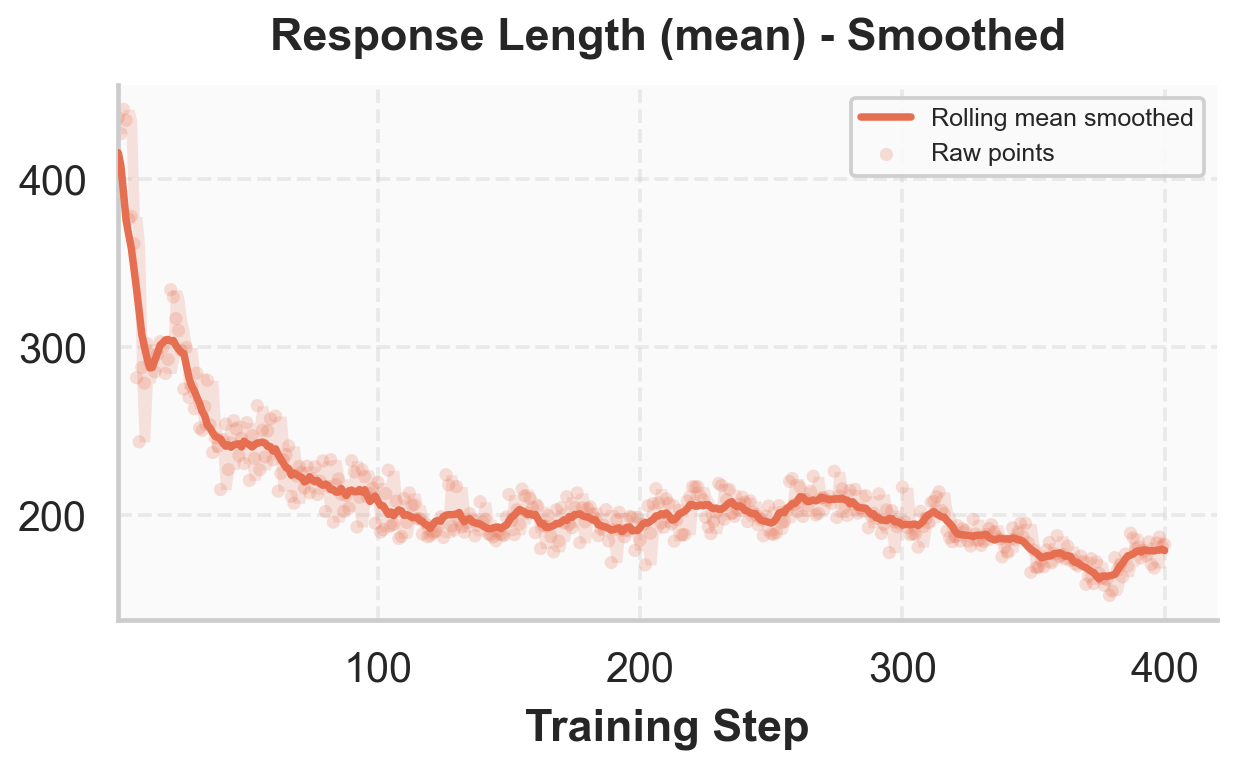} &
        \includegraphics[width=0.30\textwidth,height=0.105\textheight,keepaspectratio]{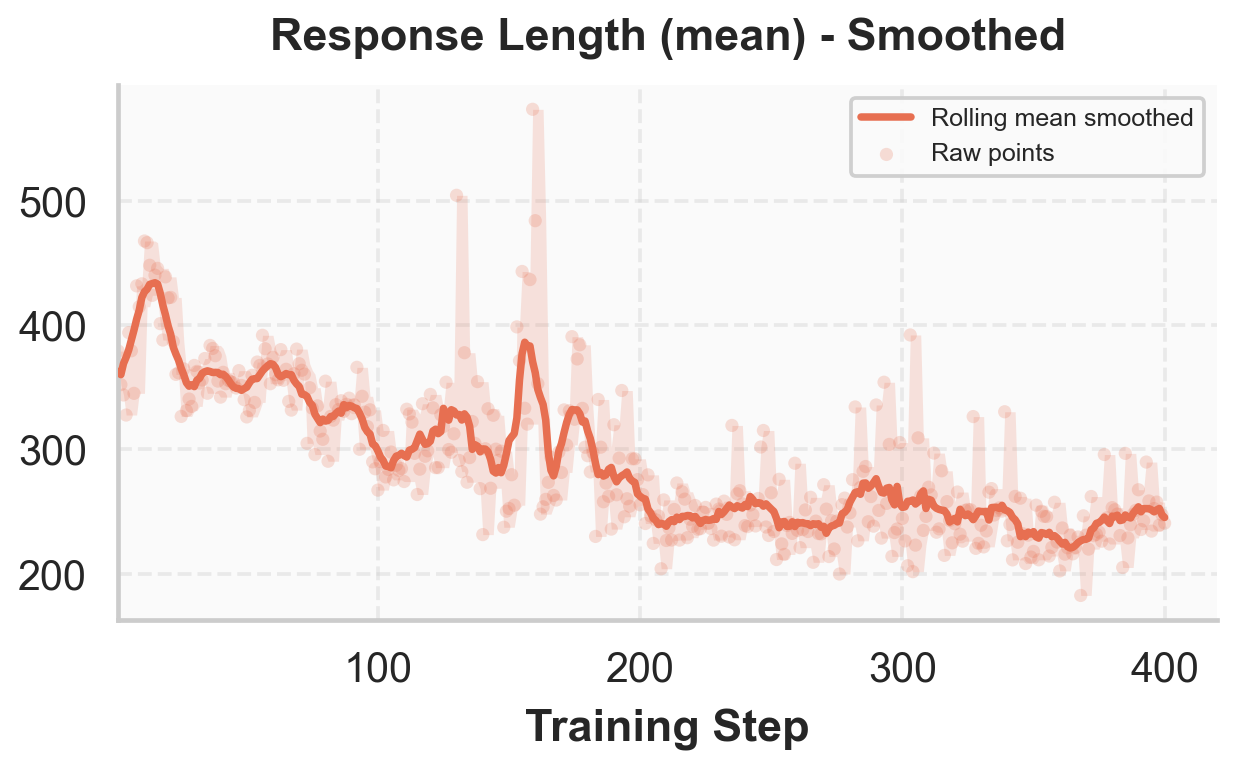} \\
        \includegraphics[width=0.30\textwidth,height=0.105\textheight,keepaspectratio]{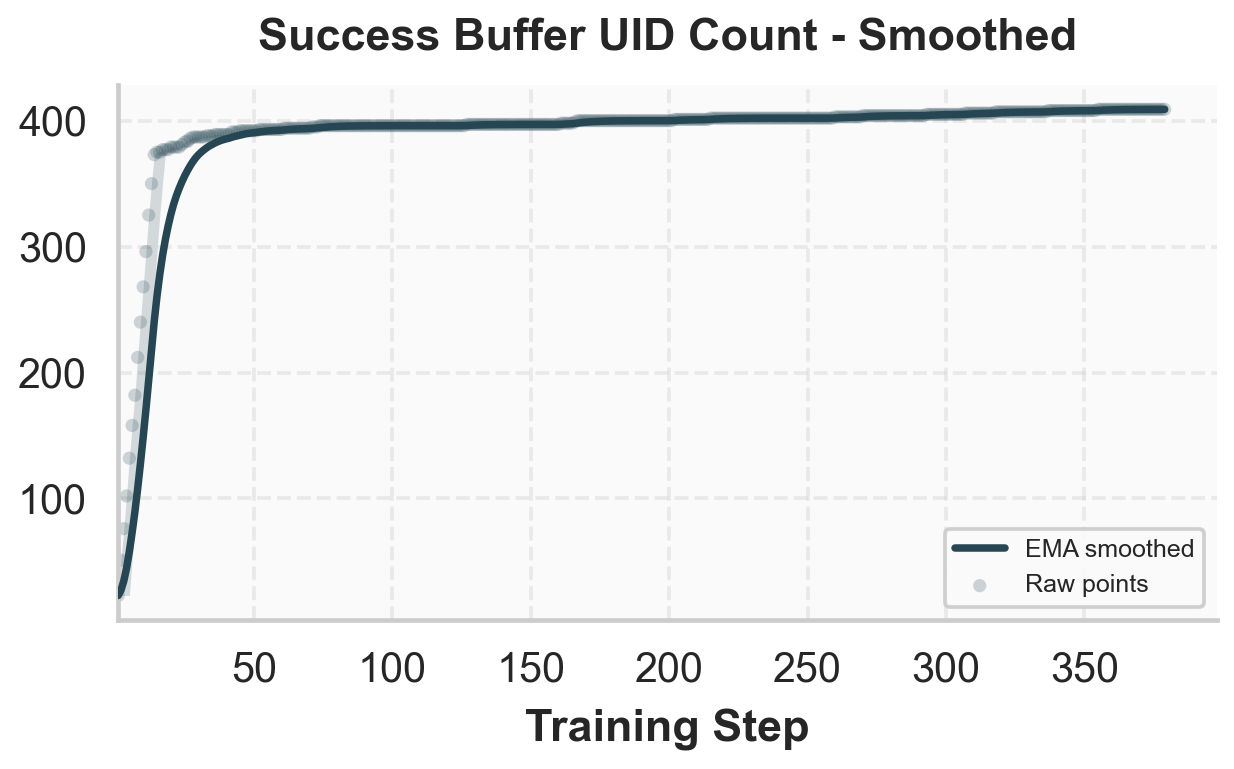} &
        \includegraphics[width=0.30\textwidth,height=0.105\textheight,keepaspectratio]{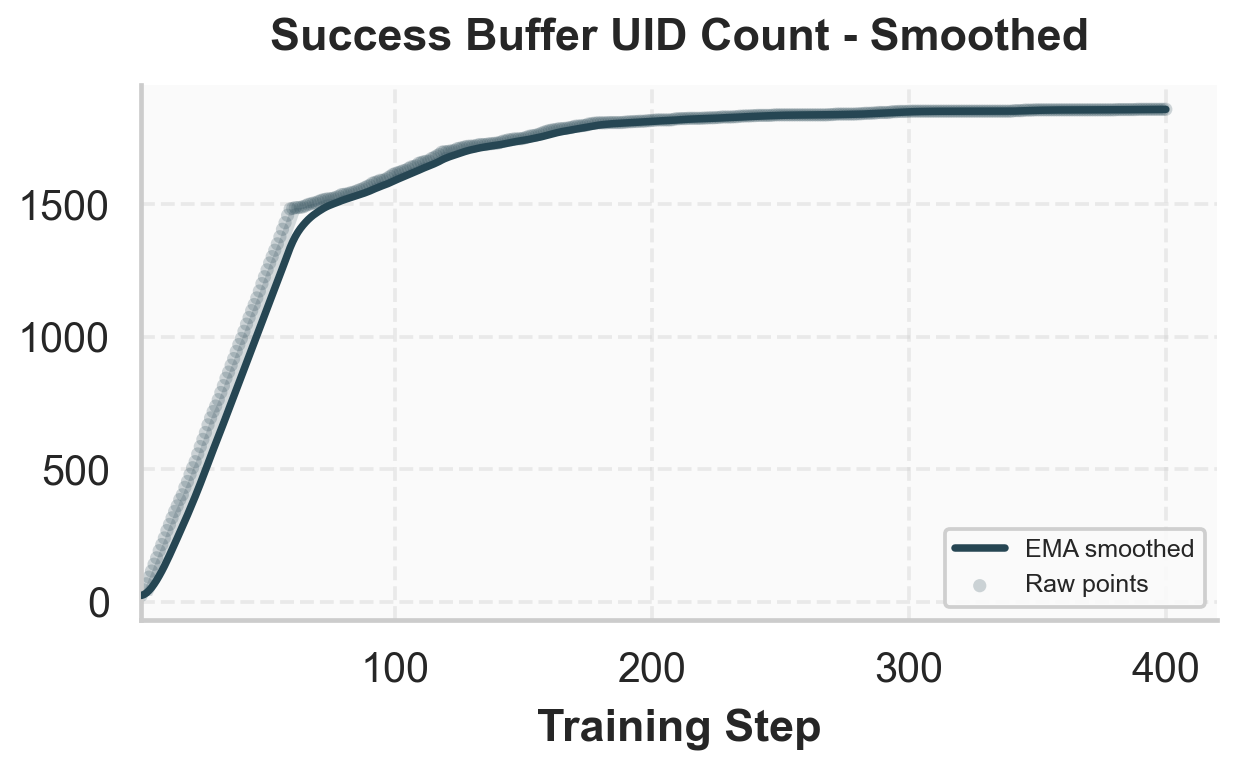} &
        \includegraphics[width=0.30\textwidth,height=0.105\textheight,keepaspectratio]{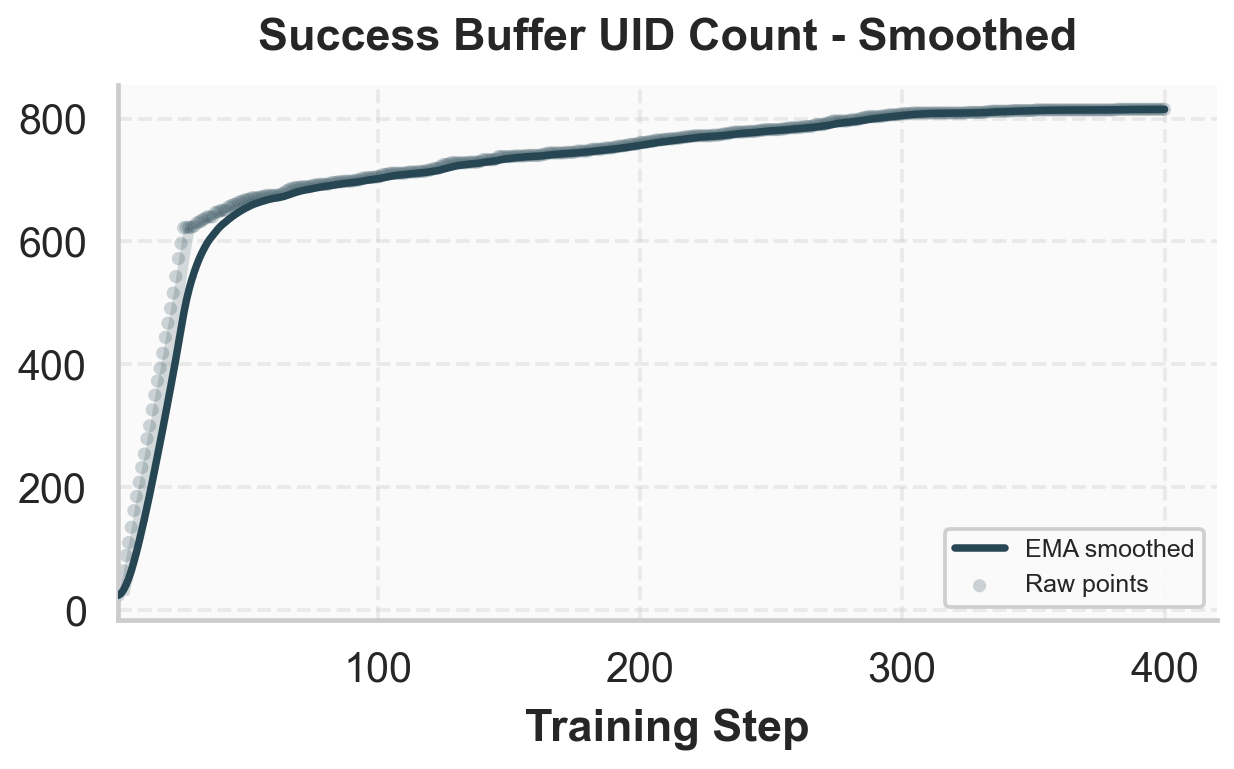} \\
        \includegraphics[width=0.30\textwidth,height=0.105\textheight,keepaspectratio]{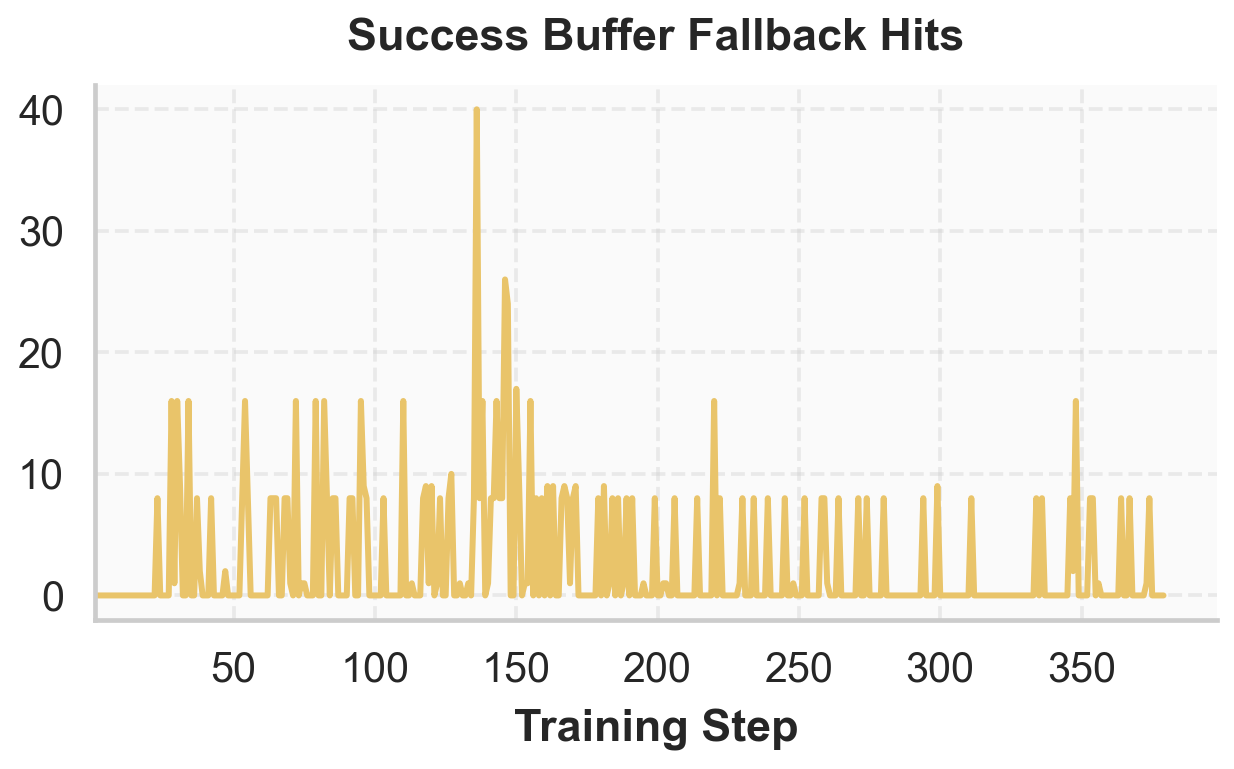} &
        \includegraphics[width=0.30\textwidth,height=0.105\textheight,keepaspectratio]{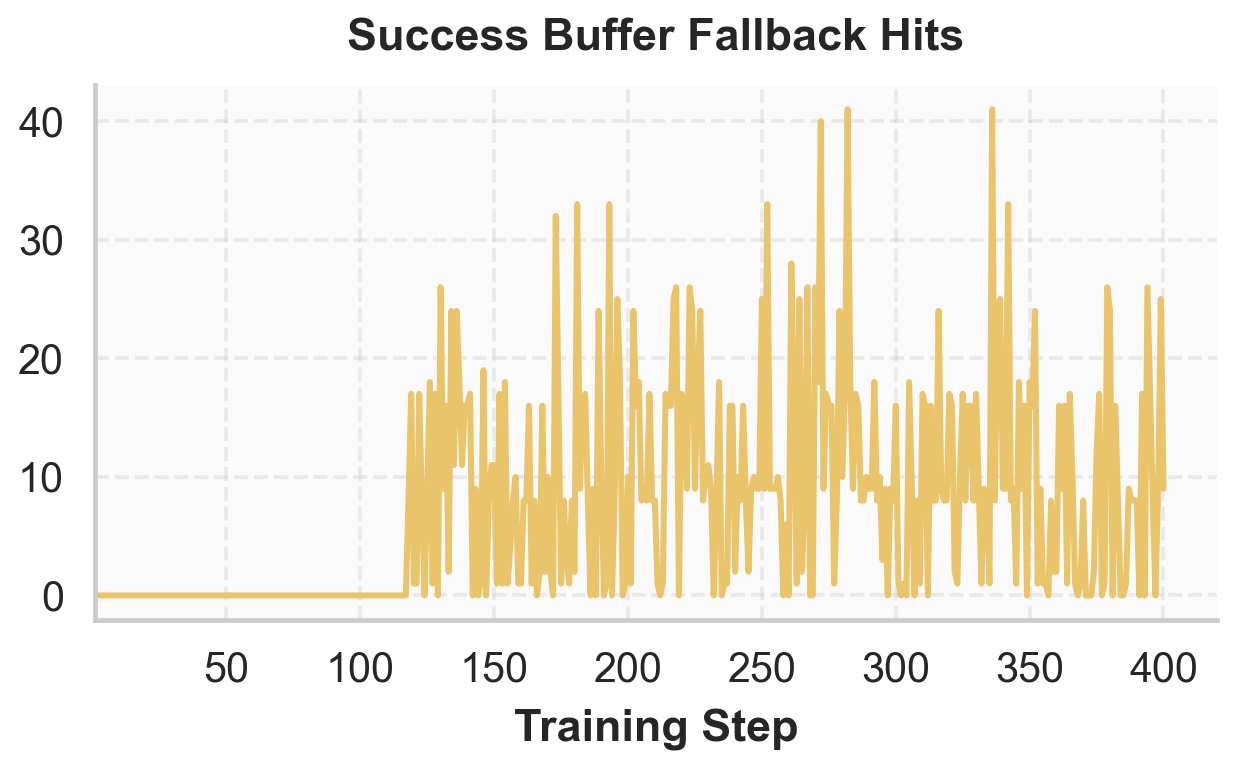} &
        \includegraphics[width=0.30\textwidth,height=0.105\textheight,keepaspectratio]{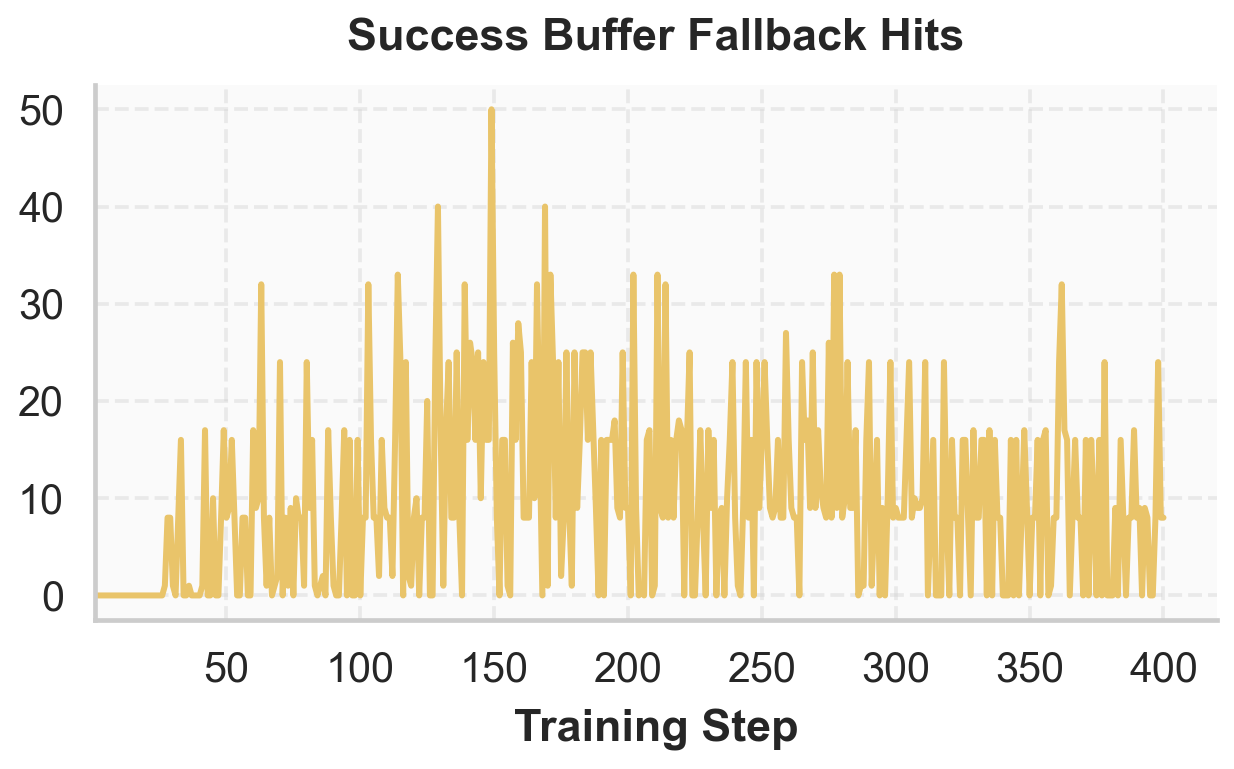} \\
        \includegraphics[width=0.30\textwidth,height=0.105\textheight,keepaspectratio]{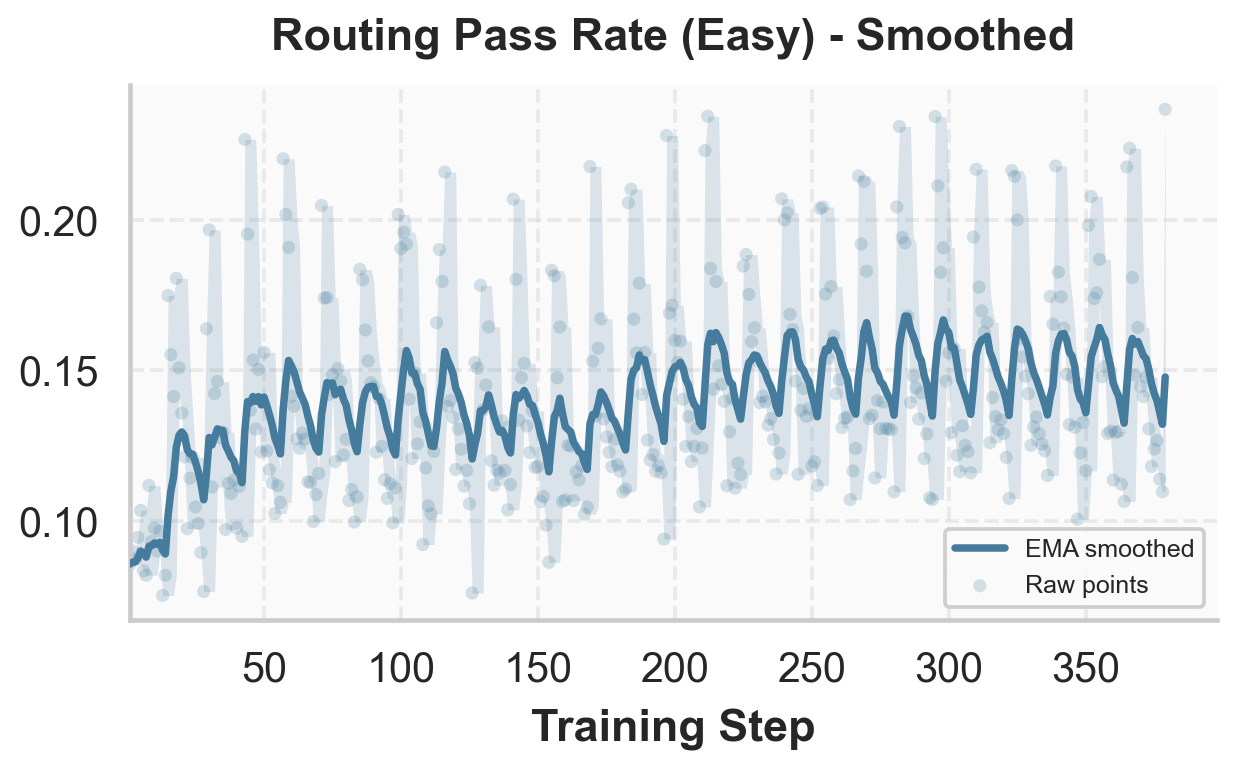} &
        \includegraphics[width=0.30\textwidth,height=0.105\textheight,keepaspectratio]{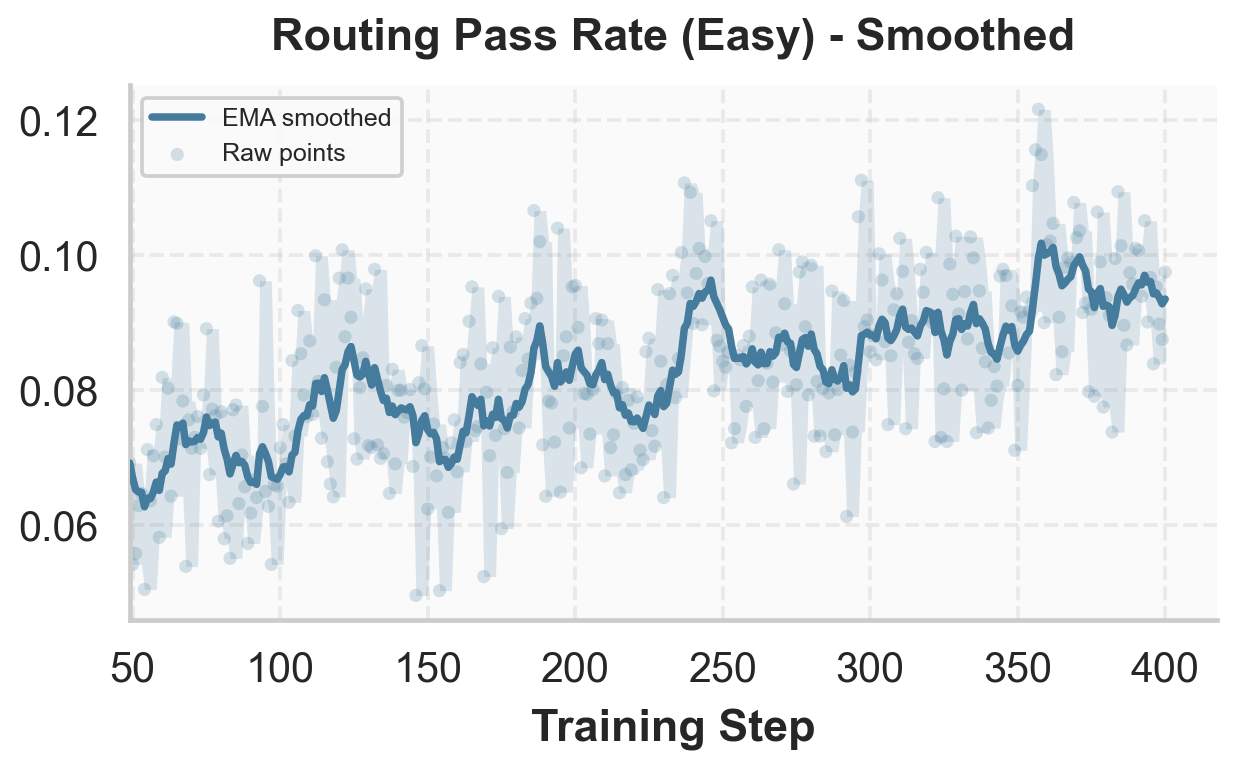} &
        \includegraphics[width=0.30\textwidth,height=0.105\textheight,keepaspectratio]{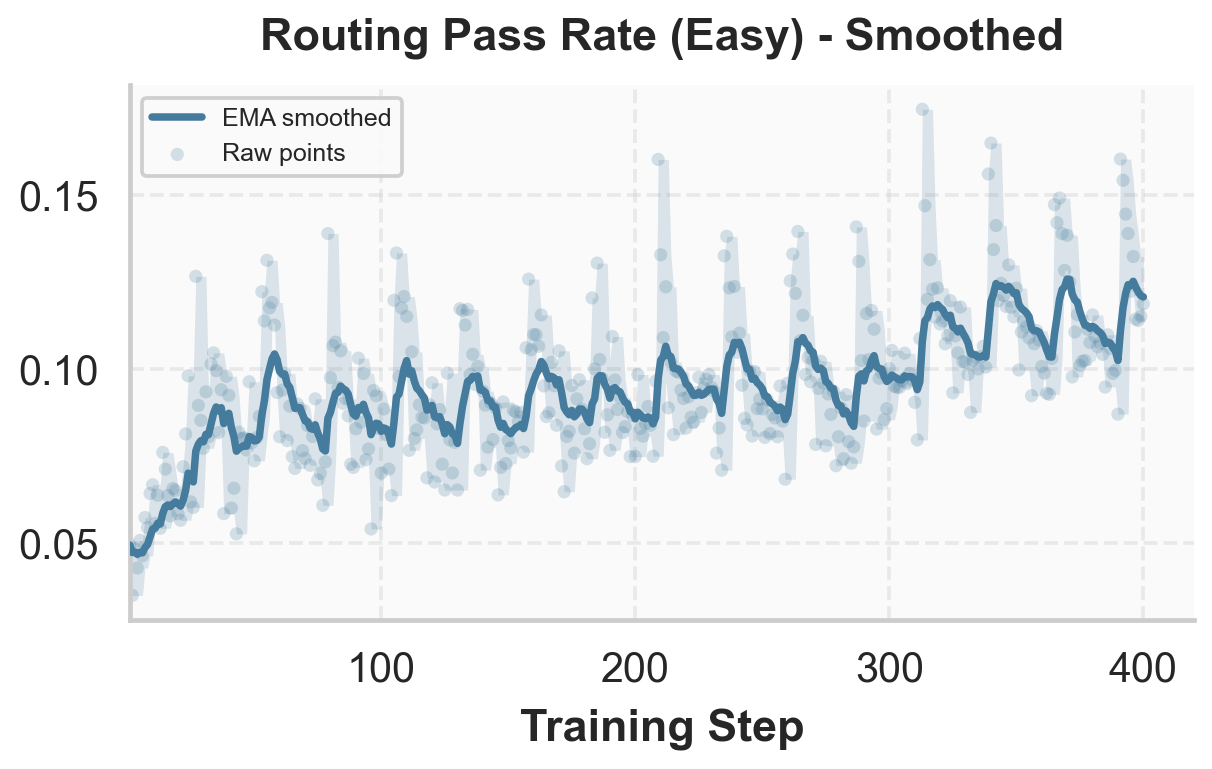} \\
        \includegraphics[width=0.30\textwidth,height=0.105\textheight,keepaspectratio]{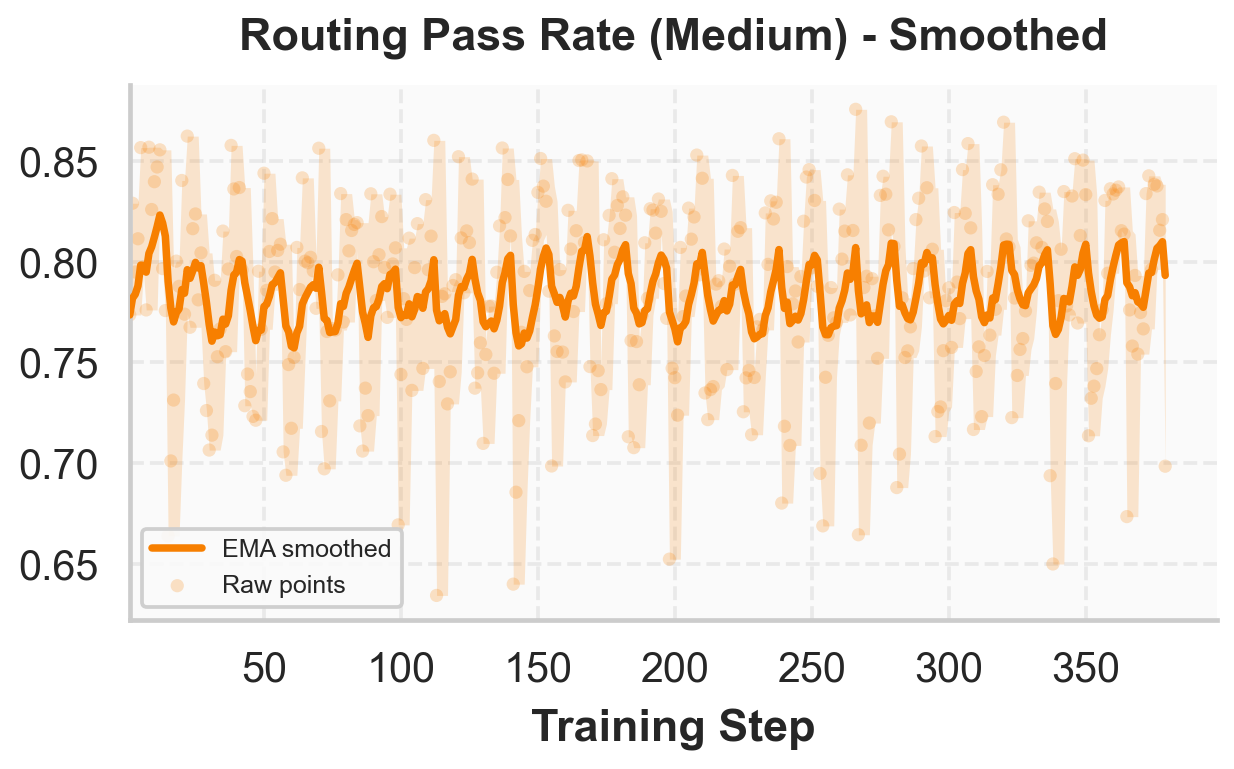} &
        \includegraphics[width=0.30\textwidth,height=0.105\textheight,keepaspectratio]{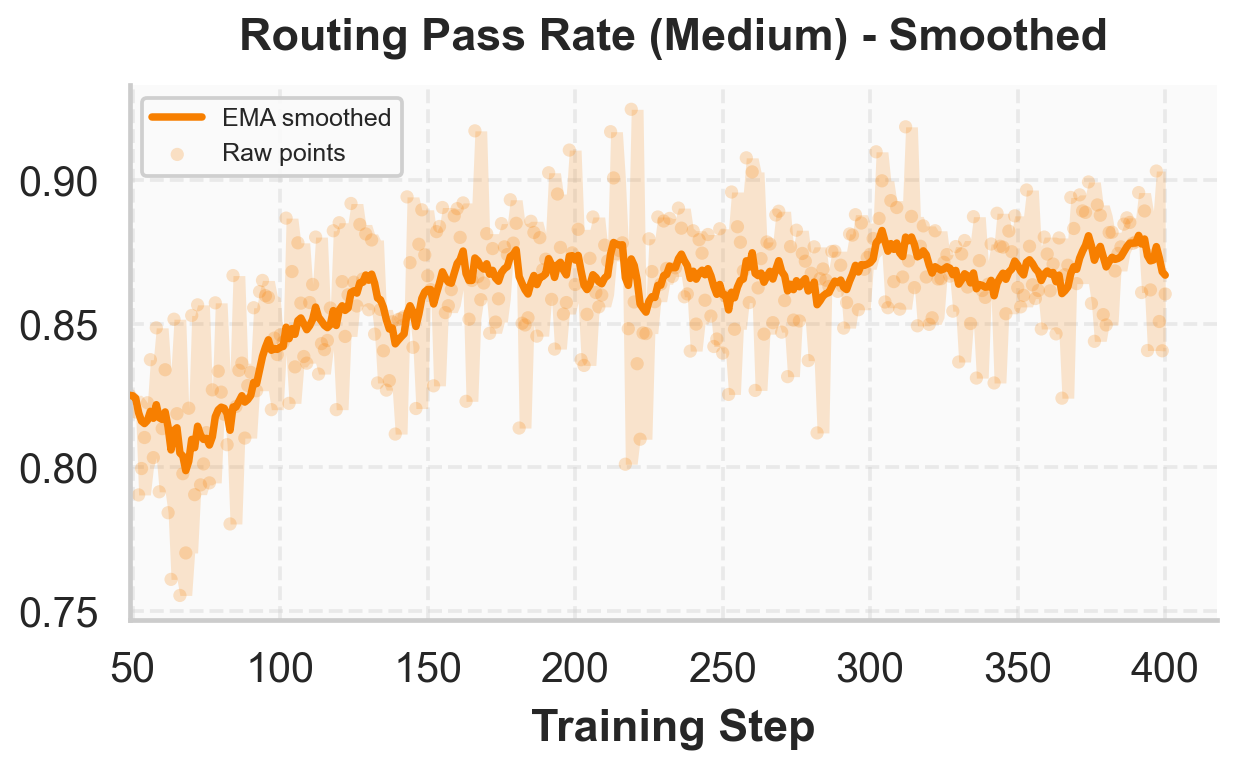} &
        \includegraphics[width=0.30\textwidth,height=0.105\textheight,keepaspectratio]{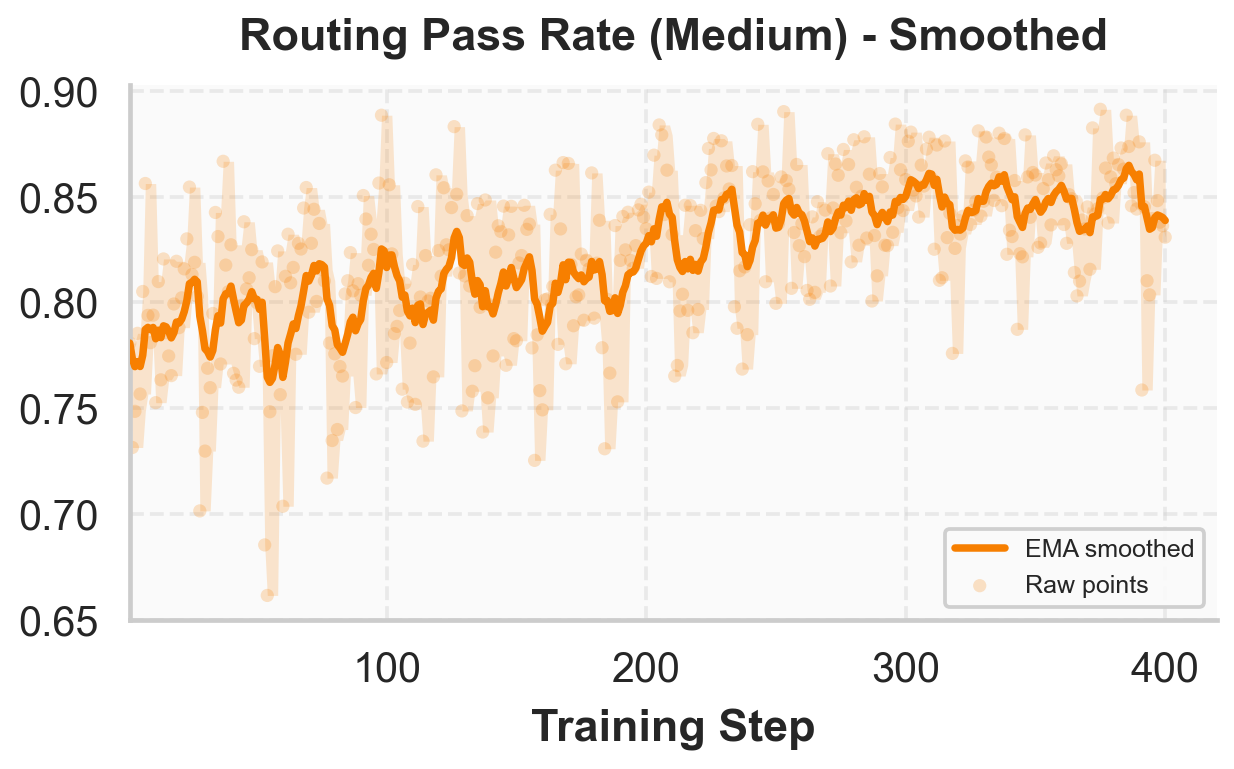} \\
        \includegraphics[width=0.30\textwidth,height=0.105\textheight,keepaspectratio]{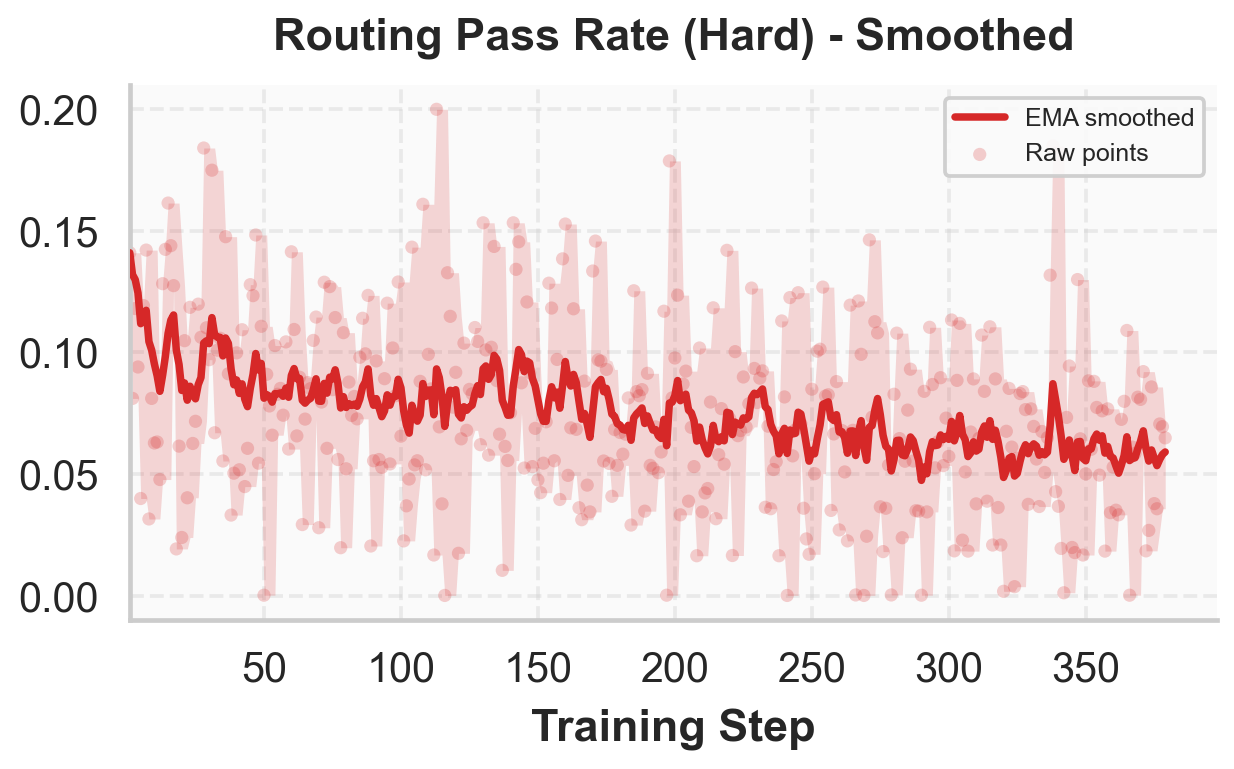} &
        \includegraphics[width=0.30\textwidth,height=0.105\textheight,keepaspectratio]{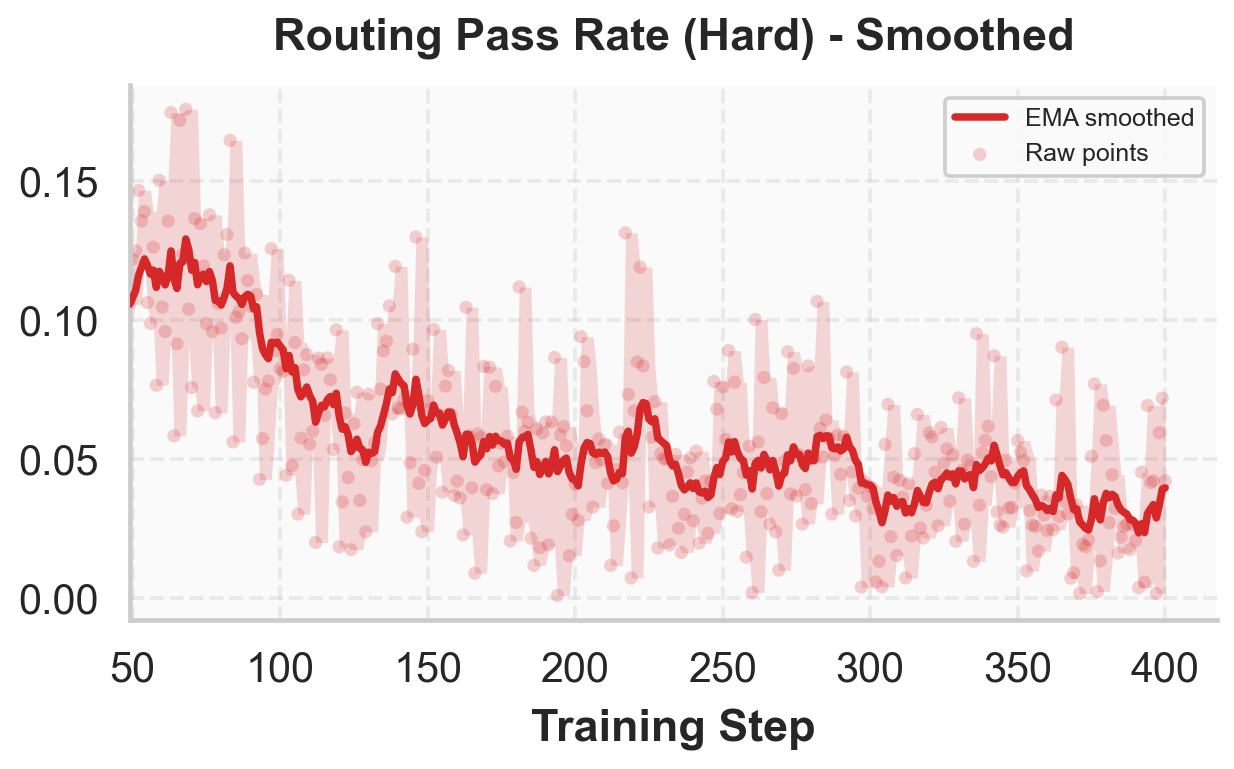} &
        \includegraphics[width=0.30\textwidth,height=0.105\textheight,keepaspectratio]{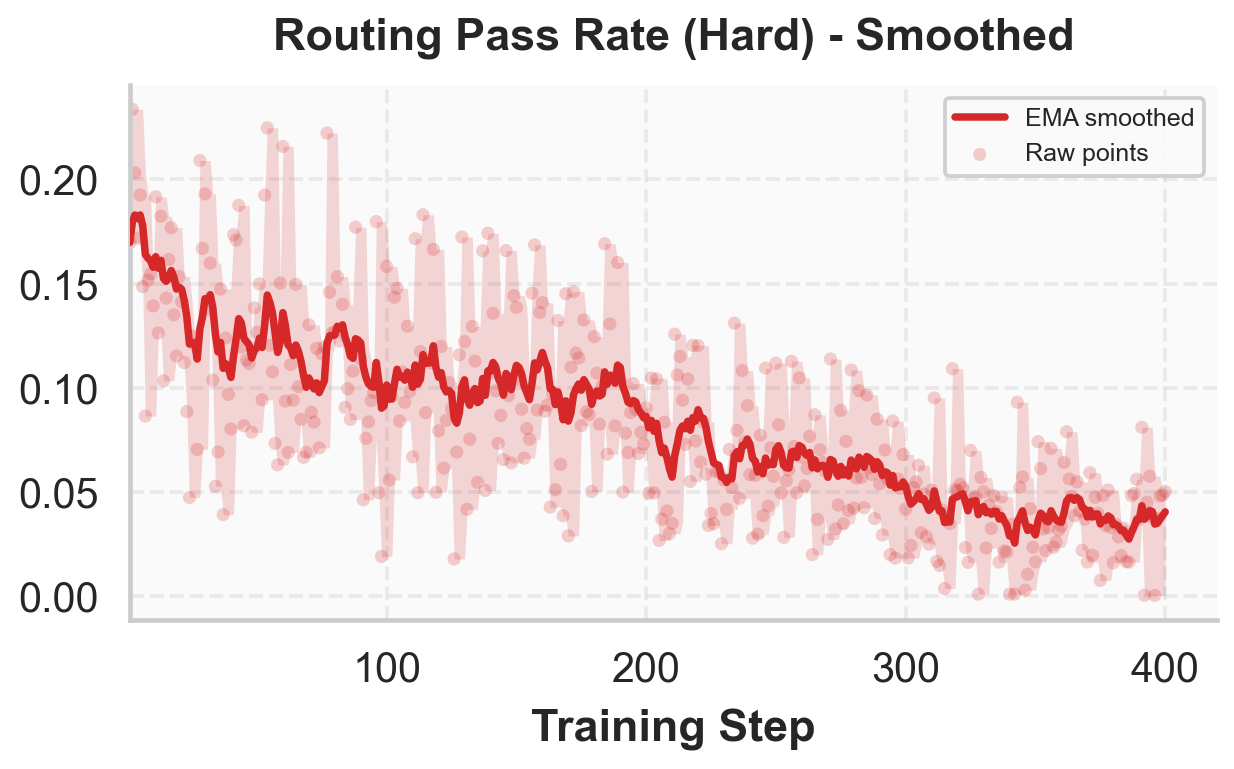}
    \end{tabular}
    \caption{\textbf{Additional training dynamics on STEM datasets.} Columns correspond to biology, chemistry, and material tasks. Rows from top to bottom show actor entropy, response length, success-buffer UID count, success-buffer fallback hit count, easy-problem fraction, medium-problem fraction, and difficult-problem fraction.}
    \label{fig:appendix_more_training_dynamics}
\end{figure*}

\end{document}